\documentclass{article}
\usepackage{multirow}
\usepackage{subfigure}
\usepackage[utf8]{inputenc}
\usepackage{amsmath}
\usepackage{hyperref}
\usepackage{url}
\usepackage{amssymb}
\usepackage{wrapfig}
\usepackage[overload]{empheq}
\usepackage{cleveref} 
\usepackage{CJK}
\usepackage{indentfirst}
\usepackage{mathtools}
\usepackage{amsmath}
\usepackage{amssymb}
\usepackage{mathtools}
\usepackage{amsthm}
\usepackage{comment}
\usepackage{color}
\usepackage{algorithm}
\usepackage{algorithmic}
\usepackage{braket}
\usepackage{relsize}
\usepackage{adjustbox}
\usepackage{lmodern}
\theoremstyle{plain}
\usepackage{float}
\usepackage{microtype}
\usepackage{graphicx}
\usepackage{subfigure}
\usepackage{booktabs} 
\usepackage{tikz}
\newcommand{\mcircled}[1]{\tikz[baseline=(char.base)]{
    \node[shape=circle,draw,inner sep=1pt, font=\scriptsize] (char) {#1};}}

\usepackage[final,nonatbib]{neurips_2024}

\usepackage[utf8]{inputenc} 
\usepackage[T1]{fontenc}    
\usepackage{hyperref}       
\usepackage{url}            
\usepackage{booktabs}       
\usepackage{amsfonts}       
\usepackage{nicefrac}       
\usepackage{microtype}      
\usepackage{xcolor}         

\theoremstyle{plain}
\newtheorem{theorem}{Theorem}[section]
\newtheorem{proposition}[theorem]{Proposition}
\newtheorem{lemma}[theorem]{Lemma}

\theoremstyle{definition}
\newtheorem{definition}[theorem]{Definition}

\theoremstyle{remark}
\newtheorem{remark}[theorem]{Remark}

\title{Understanding Representation of Deep Equilibrium Models from Neural Collapse Perspective}

\author{%
  Haixiang Sun\\
  ShanghaiTech University\\
  \texttt{sunhx@shanghaitech.edu.cn} \\
  \And
  Ye Shi \thanks{Corresponding author.} \\
  ShanghaiTech University \\
  \texttt{shiye@shanghaitech.edu.cn} \\
  }

\begin{document}
\flushbottom
\addtocontents{toc}{\protect\setcounter{tocdepth}{-1}}
\maketitle

\begin{abstract}
Deep Equilibrium Model (DEQ), which serves as a typical implicit neural network, emphasizes their memory efficiency and competitive performance compared to explicit neural networks. However, there has been relatively limited theoretical analysis on the representation of DEQ. In this paper, we utilize the Neural Collapse ($\mathcal{NC}$) as a tool to systematically analyze the representation of DEQ under both balanced and imbalanced conditions. $\mathcal{NC}$ is an interesting phenomenon in the neural network training process that characterizes the geometry of class features and classifier weights. While extensively studied in traditional explicit neural networks, the $\mathcal{NC}$ phenomenon has not received substantial attention in the context of implicit neural networks. 
We theoretically show that $\mathcal{NC}$ exists in DEQ under balanced conditions. Moreover, in imbalanced settings, despite the presence of minority collapse, DEQ demonstrated advantages over explicit neural networks. These advantages include the convergence of extracted features to the vertices of a simplex equiangular tight frame and self-duality properties under mild conditions, highlighting DEQ's superiority in handling imbalanced datasets. Finally, we validate our theoretical analyses through experiments in both balanced and imbalanced scenarios. 

\end{abstract}

\section{Introduction}
Recently, there has been significant research on \textit{implicitly-defined layers} in neural networks \cite{cvxpylayers2019, amos2017optnet, bai2019deep, blondel2022efficient, chen2018neural, gould2021deep, gu2020implicit, sun2022alternating}, where the output is implicitly mapped from the input under certain conditions. These layers embed interpretability and introduce inductive bias \cite{jacot2018neural} into black-box neural networks, demonstrating superior performance compared to existing explicit layers. 

Among these implicit networks, the Deep Equilibrium Model (DEQ) is a memory-efficient architecture that represents all hidden layers as the equilibrium point of a nonlinear fixed-point equation. Due to the absence of a closed-form solution in its forward process, DEQ can be viewed as having an infinite number of layers during iteration as long as the threshold is set low enough, enhancing its ability to fit input data. Consequently, its representational capacity is relatively stronger compared to a single-layer network structure. This phenomenon explains why DEQ has achieved state-of-the-art results in classification tasks compared to existing architectures like ResNet. For instance, it has been successfully applied to language tasks and image classification tasks, reaching state-of-the-art performance. Additionally, DEQ can be applied in various domains and integrated with numerous other models, including inverse problems \cite{gilton2021deep}, Neural ODEs \cite{ding2024two}, diffusion models \cite{huang2024efficient, pokle2022deep}, Gaussian processes \cite{gao2023wide}, and more.

However, recent research reveals a phenomenon called Neural Collapse ($\mathcal{NC}$) concerning the learned deep representations across datasets in image classification tasks \cite{papyan2020prevalence}. Under the $\mathcal{NC}$ regime, the last-layer feature of each sample in neural networks collapses to their within-class mean, and the classifier vector converges to a simplex Equiangular Tight Frame (ETF). Theoretical analyses \cite{dang2023neural, mixon2020neural, tirer2022extended} indicate that under the Unconstrained Features Mode (UFM) condition, specific features $\boldsymbol{H}^0$ can be isolated from the entire network, known as the layer-peeled model \cite{fang2021exploring}. In this scenario, Neural Collapse ($\mathcal{NC}$) is observed under certain conditions, suggesting that $\mathcal{NC}$ is agnostic to the backbone of feature extraction. Moreover, since $\mathcal{NC}$ measures the degree of proximity between features of the same category, an imbalanced dataset can exert a more negative influence on the performance of $\mathcal{NC}$. For instance, classes with fewer samples may not separate well and could converge in the same direction, leading to what is known as \textit{Minority Collapse} \cite{fang2021exploring}. Thus, the $\mathcal{NC}$ metric serves as a valuable indicator of a model's behavior in the context of imbalanced datasets.

The reasons behind the superior performance of DEQ still lack theoretical proof and comprehensive quantitative analysis. Additionally, to the best of our knowledge, no prior work has integrated DEQ with imbalanced scenarios. In our study, we integrate DEQ with layer-peeled models, add constraints with respect to weights $\boldsymbol{W}_\text{DEQ}$, and consider the results of fixed-point iteration as the output of DEQ. Therefore, we analyze the performance of $\mathcal{NC}$ in DEQ by continuously deriving the lower bound of the loss function under certain constraints, allowing us to assess how $\mathcal{NC}$ manifests in the training performance of the network. Similarly, we apply the same operations to explicit neural networks for comparison. Our results show that DEQ performs similarly to explicit neural networks under balanced settings. We further extend the dataset to imbalanced conditions and analyze the $\mathcal{NC}$ performance in DEQ, explaining why DEQ tends to outperform explicit neural networks under mild conditions. We systematically analyze performance in terms of feature convergence, distance to the Simplex ETF, and the parallel relationship between extracted features and classifier weights. These analyses uncover the reasons behind the superior performance of DEQ compared to explicit neural networks during training. Additionally, the experimental results in both balanced and imbalanced scenarios validate our theoretical analyses. 

Our main contributions are:
\begin{itemize}
    \item We systematically analyzed the representation of DEQ from the $\mathcal{NC}$ perspective and compared their performance with explicit neural networks. Our theoretical analysis shows that both DEQ and explicit neural networks exhibit the $\mathcal{NC}$ phenomenon in balanced datasets. 
        
    \item Under imbalanced settings, we theoretically proved the convergence of extracted features to the vertices of a simplex ETF and alignment with classifier weights under certain conditions, demonstrating DEQ's advantages over explicit neural networks under some mild conditions. 
    
    \item Experimental results on Cifar-10 and Cifar-100 validated our theoretical findings for distinguishing the differences between DEQ and explicit neural networks.
\end{itemize}

\section{Background and related works}

We consider a classification task with $K$ classes. Let $n_k$ denote the number of training samples in each class $k$, and $N=\sum\limits_{k=1}^K n_k$ represent the total number of training samples. A traditional neural network can be expressed as a mapping: 
\begin{equation}
\label{neuralnet}
    \psi(\boldsymbol x)= \boldsymbol W\phi(\boldsymbol x) + \boldsymbol b,
\end{equation}
where $\phi(\boldsymbol x):\mathbb{R}^{\text{in}\times N}\rightarrow\mathbb{R}^{D\times N}$ is the feature extraction, $\boldsymbol W\in\mathbb{R}^{K\times D}$ and $\boldsymbol b\in\mathbb{R}^{K}$ are the classifiers and bias in the last layer, respectively. For simplicity, we consider the bias-free case and omit the term $\boldsymbol b$. Besides, we will denote $\boldsymbol H = \phi(\boldsymbol x)$ in later sections.

\subsection{Deep Equilibrium Models}
There have been numerous neural network architectures designed for various practical tasks from different perspectives \cite{gao2023wide, lee2017deep, lorraine2020optimizing, mei2018mean, ren2021comprehensive}. DEQ, a typical implicit network \cite{el2021implicit, tsuchida2022declarative}, incorporates unrolling methods \cite{domke2012generic, monga2021algorithm}, which are devised for training arbitrarily deep networks by integrating all the network layers into one \cite{bai2019deep, bai2020multiscale, bai2021stabilizing, ling2024deep,ling2023global,xie2022optimization}. 

Let $f_\theta(\boldsymbol z,\boldsymbol x)$ represent a DEQ layer with input $\boldsymbol x$ parameterized by $\theta$. When $z^\star$ reaches the equilibrium point, it satisfies:
\begin{equation}
    g_\theta(\boldsymbol z^\star, \boldsymbol x)\triangleq f_\theta(\boldsymbol z^\star, \boldsymbol x)-\boldsymbol z^\star=0.
\end{equation}

The forward procedure mostly employs the Broyden solver \cite{broyden1965class} for iterative solving:
\begin{equation}
    \boldsymbol z_{t+1} = \boldsymbol z_t - \boldsymbol B_t^{-1} g_\theta (\boldsymbol z_t,\boldsymbol x),
\end{equation}
where $\boldsymbol B_t^{-1}$ refers to the approximation of inverse matrix $\nabla_{\boldsymbol z}^{-1} g_\theta(\boldsymbol z_t,\boldsymbol x)$, as well as the same parameter $\theta$ shared across iterations. However, the solution can be quite unstable, and efforts have been made to enhance stability and robustness \cite{li2022cerdeq,ramzi2023test,wei2022certified,winston2020monotone}. Especially, regarding the computation of the inverse matrix, it can be expanded in the form of a Neumann series \cite{geng2021training,yang2022closer}. Besides, accelerating and stabilizing the backward procedure is also an important issue in DEQ \cite{fung2022jfb}. 

\subsection{Neural Collapse \texorpdfstring{$\mathcal{NC}$}{NC}}

The phenomenon of $\mathcal{NC}$ was initially uncovered by \cite{papyan2020prevalence}, which is considered an intriguing regularity in neural network training with many elegant geometric properties \cite{thrampoulidis2022imbalance,yaras2022neural,zhu2021a}. When the model is at the terminal phase of training (TPT), or more precisely, achieves zero training error, the within-class means of features and the classifier vectors converge to the vertices of a simplex Equiangular Tight Frame (ETF) on a balanced dataset.

\begin{definition}
    (Simplex Equiangular Tight Frame) A collection of points ${\boldsymbol s}_i\in \mathbb{R}^D$, $i=1,2,\cdots, K$, is said to be a simplex equiangular tight frame if
    \begin{equation}
        {\boldsymbol S}=\alpha\sqrt{\frac{K}{K-1}}{\boldsymbol P}\Big({\boldsymbol I}_K-\frac{1}K {\boldsymbol 1}_K{\boldsymbol 1}_K^T\Big),
    \end{equation}
\end{definition}
where $\alpha$ is a non-zero scalar, ${\boldsymbol S}=[{\boldsymbol s}_1,\cdots,{\boldsymbol s}_k]\in \mathbb{R}^{D\times K}$, ${\boldsymbol I}_K \in \mathbb{R}^{K\times K}$ is the identity matrix, ${\boldsymbol 1}_K$ is the ones vector, and ${\boldsymbol P} \in \mathbb{R}^{D\times K}(D \geq K)$ is a partial orthogonal matrix such that ${\boldsymbol P}^T{\boldsymbol P} = {\boldsymbol I}_K$.

$\mathcal{NC}$ incorporates the following four properties of the last-layer features and classifiers in deep learning training on balanced datasets:

$\mathcal{NC}1$: \textbf{Variability collapse:} The feature within-class converges to a unique vector, \textit{i.e.}, for any sample $i$ in the same class $k$, its feature $\boldsymbol h_{k,i}$ satisfies $\|\boldsymbol h_{k,i}-\bar{\boldsymbol h}_k\|\rightarrow 0, k\in[k]$, with the training procedure.

$\mathcal{NC}2$: \textbf{Convergence to simplex ETF:}  The mean value $\boldsymbol h^\star$ of optimal features for each class collapses to the vertices of the simplex ETF.

$\mathcal{NC}3$: \textbf{Convergence to self-duality:} The class means and the classifier weights mutually converge: $\frac{\boldsymbol W^\star}{\|\boldsymbol W\|}=\frac{\boldsymbol H^\star}{\|\boldsymbol H\|}$.

$\mathcal{NC}4$: \textbf{Nearest Neighbor:} The classifier determines the class based on the Euclidean distances among the feature vector and the classifier weights.

\subsection{Layer-peeled model under balanced and imbalanced conditions}

Current studies often focus on the case where only the last-layer features and classifier are learnable without considering the layers in the backbone network under the assumption of Unconstrained Features Mode (UFM) \cite{zhu2021a}, which can also be referred to as the Layer-peeled Model \cite{fang2021exploring,kothapalli2023neural}.
First, we define the feasible set of parameters:
\begin{equation}
    \mathcal{C}=\left\{\boldsymbol w_k,h_{k,i}\mid \frac{1}K \sum\limits_{k=1}^K \|\boldsymbol w_k\|^2\leq E_W,\frac{1}K \sum\limits_{k=1}^K \frac{1}{n_k} \sum\limits_{i=1}^{n_k} \| \boldsymbol h_{k,i}\|^2\leq E_H\right\}.
\end{equation}
\begin{definition} (Layer-peeled Model) When $\boldsymbol H$ and $\boldsymbol W$ are the last layer classifier and weights respectively, then the optimization process of the neural network can be reformulated as:
\begin{equation}
\label{opt_prob}
    \min\limits_{\boldsymbol W, \boldsymbol H} ~~~~\frac{1}N\sum\limits_{k=1}^K\sum\limits_{i=1}^{n_k} \mathcal{L}({\boldsymbol W}\boldsymbol h_{k,i},{\boldsymbol y}_k)~~~\text{s.t.}~~\boldsymbol w_k,\boldsymbol h_{k,i}\in \mathcal{C},
\end{equation}
where $E_H$ and $E_W$ are two predefined values, $N$ refers to the total number of samples.
\end{definition}

It should be noted that all the loss functions $\mathcal{L}$ analyzed in our study are cross-entropy, as most current research focuses on this widely used deep learning classification loss function \cite{huang2017densely,lecun2015deep}. And though the optimization program is nonconvex; however, it can generally be mathematically tractable for analysis. Besides, experiments with unregularized loss function and randomly initialized gradient descent typically converge to non-collapse global minimizers \cite{tirer2022extended}.

Under UFM, most $\mathcal{NC}$ studies are based on 1-2 conventional layers of weights, However, there is also work \cite{dang2023neural,tirer2022extended} that extends it to analyze $M$ linear layers. Additionally, various studies have revealed additional characteristics of $\mathcal{NC}$, such as its impact on generalization \cite{galanti2021role,hui2022limitations,jiang2023generalized,yaras2022neural}, its influence on feature learning \cite{rangamani2023feature}, global optimality of the network \cite{zhou2022optimization,zhu2021a} and others. Therefore, $\mathcal{NC}$ is a very efficient tool to analyze the performance of neural networks.

\textbf{Imbalanced learning} 
However, $\mathcal{NC}$ will not occur under imbalanced settings generally. This phenomenon arises due to the imbalance in sample quantities, leading to challenges in adequately fitting features for certain classes. This is commonly referred to as \textit{minority collapse} \cite{cao2019learning,fang2021exploring}. As the degree of imbalance increases, it is expected that classifiers for minority classes converge. When Minority Collapse occurs, the neural network predicts equal probabilities for all minority classes, regardless of the input. 

To enhance learning performance in imbalanced scenarios \cite{zhang2023deep} and mitigate the effects of minority collapse, several methods have been proposed. \cite{fang2021exploring} introduced convex relaxation, modifying a loss function \cite{xie2023neural}, and incorporating a regularization term \cite{liu2023inducing}. The reweighted approach is also widely applied, with some studies measuring it based on sample quantities \cite{ren2018learning, yang2022inducing}. Additionally, adaptive techniques such as AutoBalance \cite{li2021autobalance} have been introduced, which incorporates a bilevel optimization framework, along with logit balance \cite{ren2020balanced,wang2023balancing,zhong2023understanding,zhu2022balanced}.

\begin{figure}
    \centering
    \includegraphics[scale=0.43
    ]{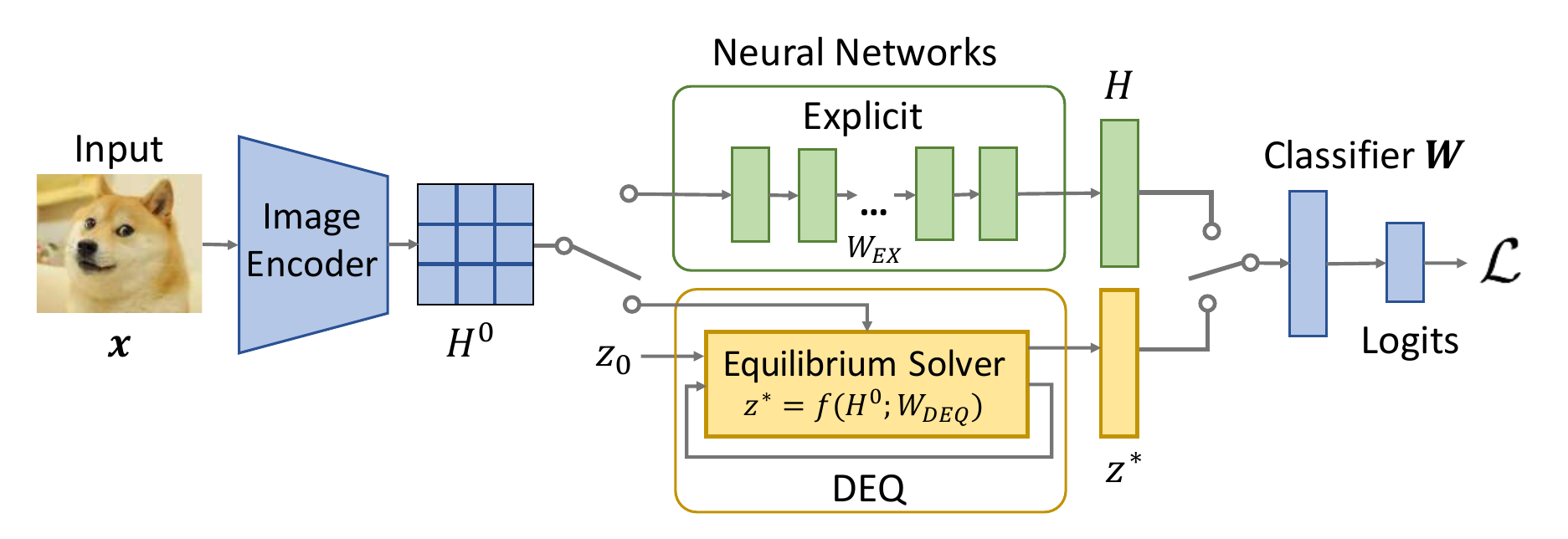}
    \caption{Illustration of feature extraction. After extracting feature maps $\boldsymbol H^0$, further features $\boldsymbol H$ or $\boldsymbol z^\star$ can be obtained by passing through an explicit neural network or DEQ. The final step involves the classifier to obtain predicted logits. To ensure a fair comparison, we standardize the backbone network and its output $\boldsymbol{H}^0$ across all conditions.}
    \label{fig_pipeline}
\end{figure}

\section{Comparison under balanced setting}

\begin{equation}
\dfrac{\partial z^\star}{}
\end{equation}
In this section, we first analyze the $\mathcal{NC}$ phenomenon in DEQ under balanced settings. As illustrated in Figure \ref{fig_pipeline}, after completing the initial feature extraction, we further examine the feature $\boldsymbol H$ obtained respectively by explicit neural networks and DEQ to reveal the $\mathcal{NC}$ phenomenon.

\subsection{\texorpdfstring{$\mathcal{NC}$}{NC} in Explicit Neural Networks}
Building upon (\ref{opt_prob}), we analyze $\mathcal{NC}$ in explicit neural networks by considering the following constrained optimization problem during training:
\begin{equation}
\label{opt_exp}
    \begin{aligned}
        \min\limits_{\boldsymbol W, \boldsymbol W_\text{EX}, \boldsymbol H^0} &~~~~\frac{1}N\sum\limits_{k=1}^K\sum\limits_{i=1}^{n} \mathcal{L}({\boldsymbol W} \boldsymbol W_\text{EX}\boldsymbol h^0_{k,i},{\boldsymbol y}_k)\\
        \text{s.t.}~~~~&\|\boldsymbol W_\text{EX}\|_F\leq E_H; \boldsymbol w_k,\boldsymbol h_{k,i}\in\mathcal{C},
    \end{aligned}
\end{equation}
where each $n_k$ is set to $n$ under the balanced setting, $\boldsymbol W_\text{EX}$ represents the subsequent network weights. For ease of comparison with DEQ, we assume that the final feature is represented as $\boldsymbol H = \boldsymbol W_\text{EX} \boldsymbol H^0$. Traditional neural network structures are nonconvex, making them challenging to analyze due to their highly interactive nature. Employing the layer-peeled model alleviates the difficulty of $\mathcal{NC}$ analysis.

\subsection{\texorpdfstring{$\mathcal{NC}$}{NC} in Deep Equilibrium models}

Building upon recent investigations into the $\mathcal{NC}$ phenomenon, we embrace the layer-peeled model, where the last-layer features $\boldsymbol h = \phi(\boldsymbol x)$ (equilibrium points in DEQ $\boldsymbol z^\star$) as unconstrained optimization variables. Accordingly, we add the following constraints to enforce $\mathcal{NC}$ in DEQ:
\begin{equation}
    \mathcal{C}_\text{DEQ}\triangleq\big\{\boldsymbol z^\star,\boldsymbol W_\text{DEQ}|~\boldsymbol z^\star=f(\boldsymbol H^0; \boldsymbol W_\text{DEQ}),~~\boldsymbol \|\boldsymbol W_\text{DEQ}\|_F\leq E_H\big\}.
\end{equation}
Compared to explicit layers, the active parameter in Deep Equilibrium models is $W_\text{DEQ}$, hence imposing restrictions on it to align with the same feasible space. Then the formulation of DEQ with $\mathcal{NC}$ becomes:
\begin{equation}
\label{DEQ_NC}
    \begin{aligned}
        \min\limits_{\boldsymbol W,\boldsymbol W_\text{DEQ}, \boldsymbol z^\star,\boldsymbol H^0} ~~~~&\frac{1}N\sum\limits_{k=1}^K\sum\limits_{i=1}^{n_k} \mathcal{L}({\boldsymbol W}\boldsymbol z^\star,{\boldsymbol y}_k)\\
        \text{s.t.}~~~~&\boldsymbol w_k,\boldsymbol h_{k,i}\in\mathcal{C}; \boldsymbol z^\star,\boldsymbol W_\text{DEQ} \in \mathcal{C}_\text{DEQ}.
    \end{aligned}
\end{equation}

No matter whether under DEQ or explicit neural networks, these constraints must be imposed. This is because when these constraints are satisfied and the loss function reaches its lower bound, the $\mathcal{NC}$ phenomenon is guaranteed. In our theoretical analysis, we assume that the DEQ is linear, that is, $\boldsymbol z^\star=\text{fixed-point}(f_\theta(\boldsymbol x),\boldsymbol z) = \sum\limits_{i=0}^\infty\boldsymbol W^i_\text{DEQ} \boldsymbol x$. Detailed analysis incorporating these constraints is provided in Appendix \ref{Supp_sec A}. 

The following theorem elucidates the specific scenarios in which the $\mathcal{NC}$ phenomenon occurs. For a fair comparison, we assume that the extracted features $\boldsymbol H^0$ of the image encoder are the same in the derivation.

\begin{theorem}({Feature collapse of explicit fully connected layers and implicit deep equilibrium models under balanced setting})\label{theorem_nc_balance}
    Suppose (\ref{opt_exp}) and (\ref{DEQ_NC}) reaches its minimal, then
    
    $\mathcal{NC}1$: For $\forall~ k=1,2,\cdots,K$ and $\forall~ i=1,2,\cdots,n$:
    \begin{equation*}
        \boldsymbol W_\text{EX}\boldsymbol h^0_{k,i}=\boldsymbol W_\text{EX}\boldsymbol h^0_{k},
    \end{equation*}
    where $\boldsymbol h^0_{k}=\sum\limits_{i\in\pi(k)} \boldsymbol h^0_{k,i}$. Similarly, if the model is DEQ, then
    \begin{equation*}
        f(\boldsymbol h^0_{k,i}; \boldsymbol W_\text{DEQ}) = f(\boldsymbol h^0_{k}; \boldsymbol W_\text{DEQ}).
    \end{equation*}

    $\mathcal{NC}2$: The classifier aligns to the Simplex ETF, regardless of whether explicit neural network and DEQ are applied:
    \begin{equation*}
        \begin{aligned}
            \boldsymbol W\boldsymbol W^T &= \sqrt{E_W/E_H}\boldsymbol W \boldsymbol W_\text{EX}\boldsymbol H^0 \\
            &= \sqrt{E_W/E_H}\boldsymbol W f(\boldsymbol H^0; \boldsymbol W_\text{DEQ}) \\
            &= \dfrac{KE_W}{K-1}\left(\boldsymbol 1_K -\frac{1}K\boldsymbol 1_K\boldsymbol 1_K^T\right).
        \end{aligned}
    \end{equation*}

    $\mathcal{NC}3$: For $\forall~ k=1,2,\cdots,K$, the feature aligns to the weights:
    \begin{equation*}
        \boldsymbol W_\text{EX} \boldsymbol h_k^0\propto \boldsymbol W_k.
    \end{equation*}
    In DEQ cases:
    \begin{equation*}
        f(\boldsymbol h^0_{k}; \boldsymbol W_\text{DEQ})\propto \boldsymbol W_k.
    \end{equation*}
\end{theorem} 
The theorem demonstrates that when the network training reaches its limit, i.e., when the loss function reaches its minimum, the $\mathcal{NC}$ phenomenon emerges regardless of whether the chosen network is DEQ or explicit neural network. Besides, in certain scenarios, the lower bound of the loss function for DEQ is relatively smaller compared to explicit neural networks. More detailed proofs are in Appendix Section \ref{Supp_sec A}.

\section{Comparison under imbalanced setting}

In this section, we analyze the performance differences between DEQ and explicit neural network on imbalanced datasets. We observe that, unlike in balanced scenarios, as long as certain conditions are met, the advantages of DEQ over explicit neural network become more pronounced on imbalanced datasets. And we provide theoretical evidence to support this phenomenon.

Suppose the total number of classes is $K$, with $K_A$ being the number of majority classes and $K_B = K - K_A$ being the number of minority classes. Each majority class has $n_A$ samples, and each minority class has $n_B$ samples. The total number of samples is given by $N = K_A n_A + K_B n_B$. Note that $n_A>n_B$ with no requirement for $K_A$ to be greater than $K_B$. We first start with the loss function, which can be partitioned into two components as follows:
\begin{equation}
\label{imb_loss}  
\begin{aligned}
    &\min\limits_{\boldsymbol W,\tilde{\boldsymbol W}, \boldsymbol H^0} \frac{K_An_A}N\sum\limits_{k=1}^{K_A}\sum\limits_{i=1}^{n_A} \mathcal{L}({\boldsymbol W}\tilde{\boldsymbol W}\boldsymbol H^0,{\boldsymbol y}_k) + \frac{K_Bn_B}N\sum\limits_{k=K_A+1}^{K_B}\sum\limits_{i=1}^{n_B} \mathcal{L}({\boldsymbol W}\tilde{\boldsymbol W}\boldsymbol H^0,{\boldsymbol y}_k),\\
    &\quad\text{s.t.}~~~ \tilde{\boldsymbol W}\in \left\{\mathcal{C}_\text{EX} ~~\text{or} ~~\mathcal{C}_\text{DEQ}\right\},~~\boldsymbol w_k,\boldsymbol h_{k,i}\in\mathcal{C},
\end{aligned}
\end{equation}
where $\tilde{\boldsymbol W}$ represents the weights of Deep Equilibrium Models $\boldsymbol W_\text{DEQ}$ and explicit neural network $\boldsymbol W_\text{EX}$. To analyze the $\mathcal{NC}$ phenomenon, we present the results in the following theorem:
\begin{theorem}
\label{thm_imbalance}
    (Neural Collapse under imbalanced settings on explicit neural networks and deep equilibrium models)

    When the loss function reaches the minimum, then
    
    $\mathcal{NC}1$: For $\forall~ k=1,2,\cdots,K$ and $\forall~ i=1,2,\cdots,n$:
    \begin{equation*}
        \boldsymbol W_\text{EX}\boldsymbol h^0_{k,i}=\boldsymbol W_\text{EX}\boldsymbol h^0_{k},
    \end{equation*}
    where $\boldsymbol h^0_{k}=\sum\limits_{i\in\pi(k)} \boldsymbol h^0_{k,i}$.  Similarly, if the model is DEQ, then
    \begin{equation*}
        f(\boldsymbol h^0_{k,i}; \boldsymbol W_\text{DEQ}) = f(\boldsymbol h^0_{k}; \boldsymbol W_\text{DEQ}).
    \end{equation*}
    
    $\mathcal{NC}2$: Not exists, but the results of explicit neural network and DEQ can be compared:
    
    Here we denote $\left(\boldsymbol h^0_k\right)^T\boldsymbol h^0_{k'}=\boldsymbol m_{k,k'}$ and $\boldsymbol S$ is a $K$-Simplex ETF, if
    \begin{equation*}
    E_H < 2\boldsymbol S_{ij} - \boldsymbol m_{ij} < \frac{1}{1-E_H}
    \end{equation*}
    is satisfied, the following inequality
        \begin{equation*}
            \left\|\left(\boldsymbol W_\text{EX}\boldsymbol H^0\right)^T\left(\boldsymbol W_\text{EX}\boldsymbol H^0\right)  - \boldsymbol S\right\|_F>\left\|f^T(\boldsymbol H^0; \boldsymbol W_\text{DEQ})f(\boldsymbol H^0; \boldsymbol W_\text{DEQ}) - \boldsymbol S\right\|_F
        \end{equation*}        
    holds.
    
    $\mathcal{NC}3$: Similarly as $\mathcal{NC}2$, though it does not exist, the results can still be compared, when
    \begin{equation*}
        \frac{E_H}{E_w+E_H}+E_H(1-E_H)<2
    \end{equation*}
    is satisfied, then the cosine distance satisfies:
    \begin{equation*}
        \cos\left( f(\boldsymbol h_{k}; \boldsymbol W_\text{DEQ}),\boldsymbol w_k\right)/\cos \left(\boldsymbol W_\text{EX}\boldsymbol h_k,\boldsymbol w_k\right)>1.
    \end{equation*}
\end{theorem}

The detailed proof is in Appendix Section \ref{supp_imb}.

Besides, the conclusion regarding the loss function is quite similar to that of Theorem \ref{thm_balance} under balanced settings. As analyzed in (\ref{fc_minhk}) and (\ref{deq_minhk}) in the Appendix, the lower bound of the loss function in DEQ is still lower than that in explicit neural network, where the performance of learned features is more evident in Figure \ref{fig_tsne}, where we use t-SNE \cite{vandermaaten08a} and Gram matrix of features to describe the performance of two models. Although the phenomenon of $\mathcal{NC}2$ and $\mathcal{NC}3$ does not exist, we have discovered in Theorem \ref{thm_imbalance} that under mild conditions, DEQ is superior in terms of $\mathcal{NC}$ compared to explicit neural network. Notably, the conditions are easy to satisfy since $E_H$ is generally very small in practice.

\begin{figure}[htbp]
\centering  
\subfigure[t-SNE results.]{
\includegraphics[width=0.45\textwidth]{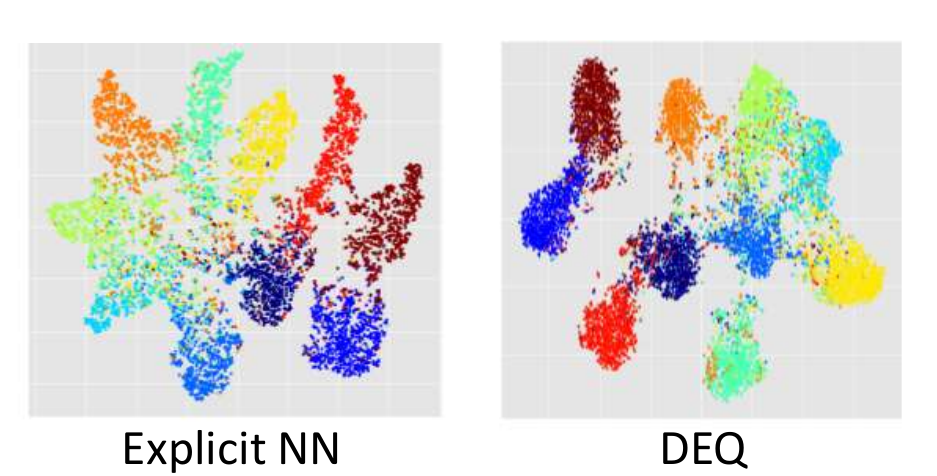}}
\hspace{10pt}
\subfigure[Visualization of the Gram matrix $\boldsymbol{H}\boldsymbol{H}^T$.]{
\includegraphics[width=0.45\textwidth]{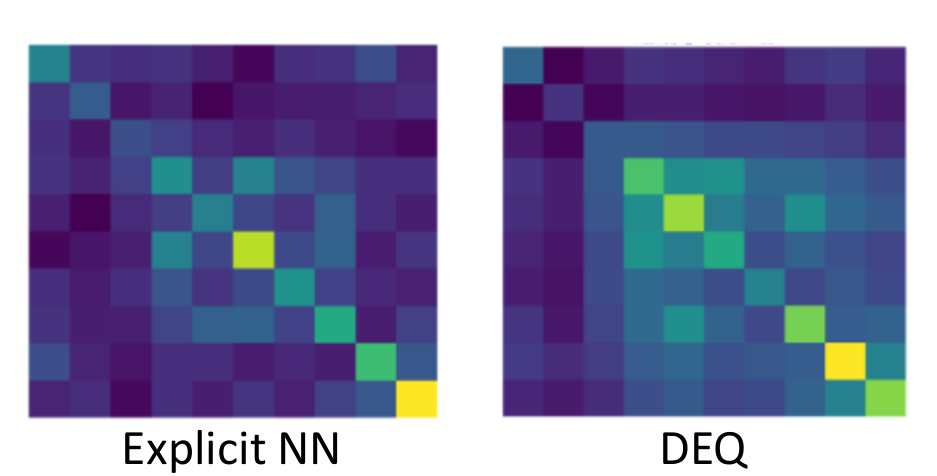}}
\caption{Under the imbalanced setting for CIFAR-10 with $K_A=3$ and $R=10$, the disparity in the learned features between Explicit Neural Networks (left) and DEQ (right).}
\label{fig_tsne}
\end{figure}

A crucial insight is that since DEQ undergoes multiple rounds of parameter adjustments for learning, it can be viewed as having an infinite number of layers, thus possessing greater representational capacity. As the network deepens, the iterative process of forward fixed-point may not necessarily reach the lowest threshold. Therefore, DEQ exhibits a certain degree of generalization for features in the minority class. Given the substantial feature differences among classes under an imbalanced dataset, the learned features by DEQ may demonstrate better adaptability to unseen categories. Consequently, compared to explicit neural network, DEQ tends to enhance performance.

{

Besides, due to the repeated iterations in solving the fixed-point iteration for some samples in the minority class with a small sample size, the model somewhat engages in multiple learning iterations for the features of samples in this class. This mitigates the impact of imbalanced samples to some extent. However, despite some improvements compared to the explicit neural network, DEQ still faces the issue of minority collapse. This conclusion is further validated in our subsequent experiments. Besides, to further discuss the situation of the dataset in terms of the degree of imbalance, we derived the following proposition:

\begin{proposition}
    \label{prop_1}
    Denote $R={K_An_A}/{N}$. When the number of samples in the majority class becomes extremely large, i.e., $R\rightarrow 1$, the features of the two kinds of classes will become:

    Majority classes: 
    \begin{equation*}
    \begin{aligned}
        \boldsymbol W_\text{EX}\boldsymbol h^0_{k,i}&=\boldsymbol W_\text{EX}\boldsymbol h^0_{k},\\
         f(\boldsymbol h^0_{k,i}; \boldsymbol W_\text{DEQ}) &= f(\boldsymbol h^0_{k}; \boldsymbol W_\text{DEQ}),
    \end{aligned}
    \end{equation*}    
    where $1\leq k\leq K_A$ and $i\in\pi(k)$. Each feature collapses to $K_A$-Simplex ETF.
    
    Minority classes:     
    \begin{equation*}
    \begin{aligned}
        \boldsymbol w_k&=\boldsymbol 0,\\
        \boldsymbol W_\text{EX}\boldsymbol h^0_{k,i}&=f(\boldsymbol h^0_{k,i}; \boldsymbol W_\text{DEQ})=\boldsymbol 0,
    \end{aligned}  
    \end{equation*}
    where $K_A+1\leq k\leq K$ and $i\in\pi(k)$.
    
    Here, $\pi(k)$ refers to the samples that belong to the class $k$.
\end{proposition}

This situation is equivalent to having a balanced dataset in the majority class, while the minority class, due to its extremely small sample size, contributes almost nothing. In such an extreme scenario, the $\mathcal{NC}$ performance of DEQ and the fully connected layer is nearly indistinguishable similar to Theorem \ref{theorem_nc_balance}. Both collapse on majority classes, resulting in a lack of learning features from minority classes meeting the results of (\ref{condition_major}) and (\ref{condition_minor}). This aligns with the findings in \cite{fang2021exploring}, where they provide more specific bounds on the ratio $K_A/K_B$ in their Theorem 5.

\section{Experiments}
\begin{figure*}
\label{fig}
  \centering
  \subfigure[Balanced dataset]{\label{fig_balance}\includegraphics[width=0.49\textwidth]{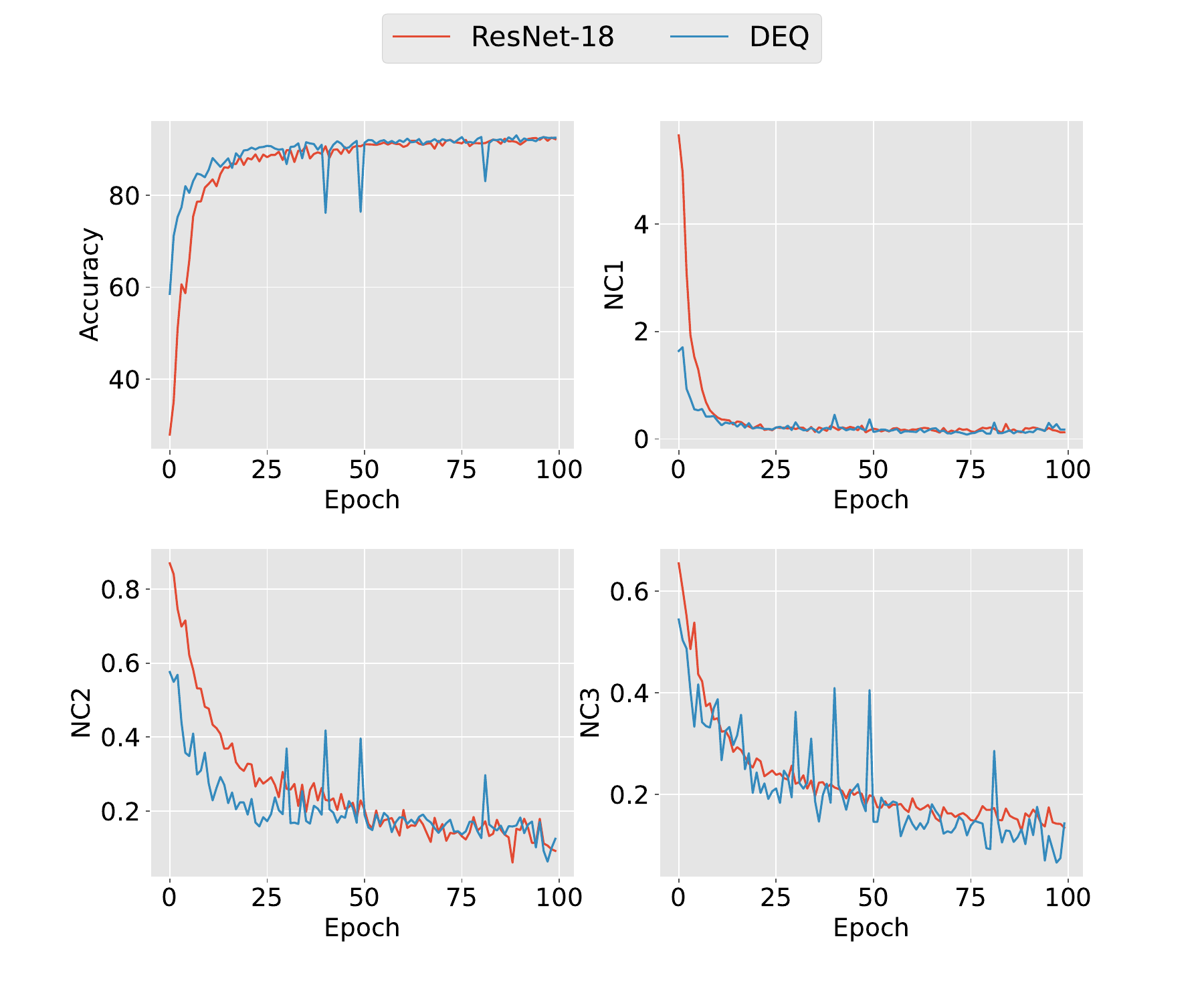}}
  \subfigure[Imbalanced dataset with $K_A=K_B=5$, $R=100$]{\label{fig_imbalance}\includegraphics[width=0.49\textwidth]{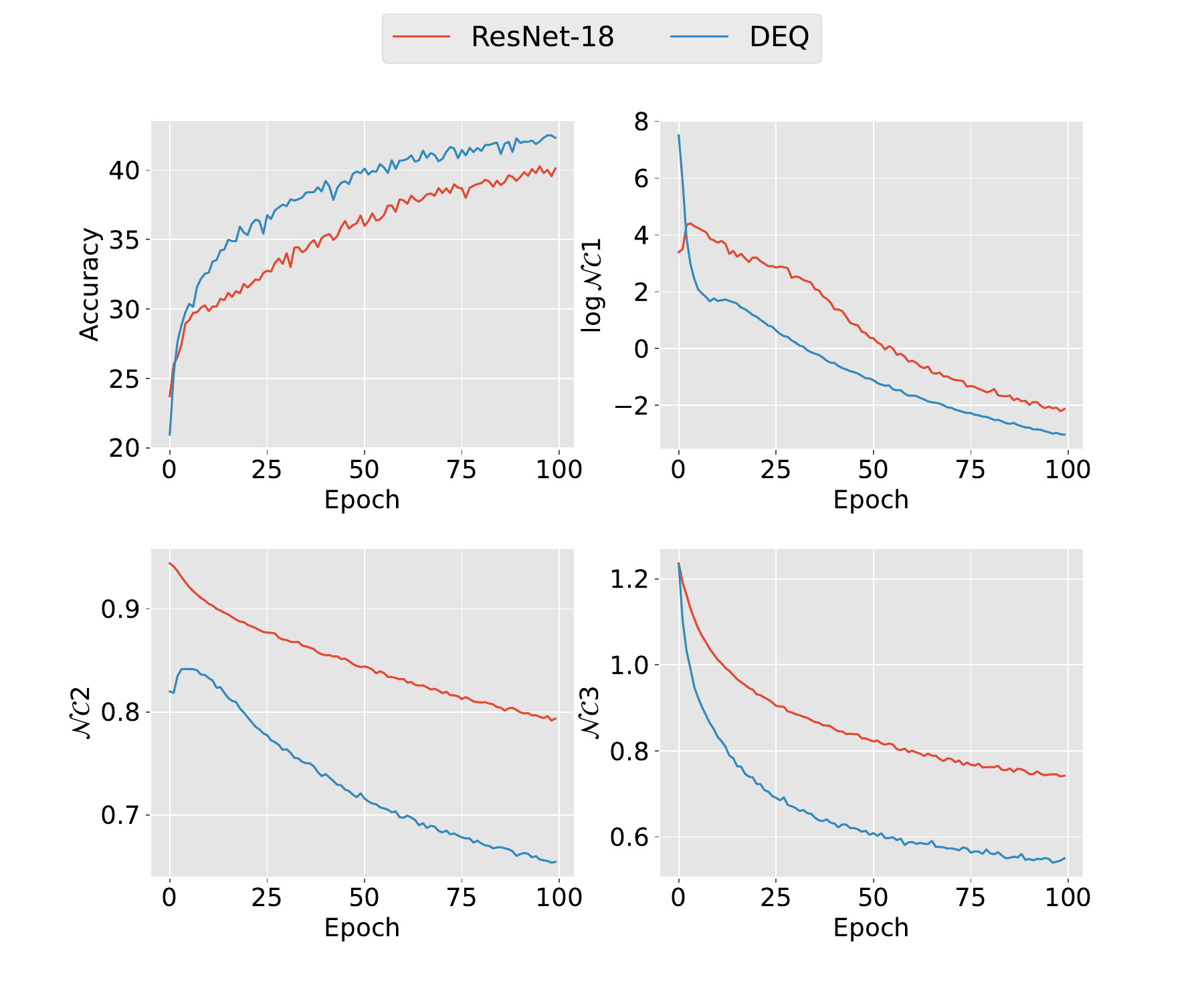}}

  \caption{Comparison of accuracy and $\mathcal{NC}$ phenomenon in training Cifar-10 dataset}
  \vspace{-10pt}
  
\end{figure*}

In this section, we empirically conducted experiments to validate the correctness of the proposed theorems. Initially, we implemented DEQ on a balanced dataset and compared its $\mathcal{NC}$ performance with that of ResNet. Subsequently, for imbalanced datasets, we tested varying degrees of imbalance by manipulating the quantities of $n_A$ and $n_B$, as well as $K_A$ and $K_B$. The experimental results showed that, on imbalanced datasets, DEQ outperformed Explicit Neural Networks. This finding is consistent with the results reported in \cite{bai2020multiscale}. All experiments were implemented using PyTorch on NVIDIA Tesla A40 48GB.

\subsection{Experiment setup}
Without loss of generality, since any traditional neural network can be formulated as a DEQ, we use ResNet18 \cite{he2016deep} as the backbone architecture here. As discussed earlier, to utilize the fixed point $\boldsymbol z^\star$ learned by DEQ as the extracted feature, we formulate the last ResNet block into a DEQ format, while maintaining the remaining structure identical to ResNet. As mentioned in \cite{bai2021stabilizing}, training with DEQ can lead to instability issues. This is especially noticeable as training progresses, where some samples struggle to converge to a fixed point. To address this, in accordance with their setting, we implement the solver with a threshold $\epsilon$ set to $10^{-3}$ and introduce an early stopping mechanism. If convergence is not achieved within $T > 20$ iterations, we terminate the fixed-point iteration. Additionally, when facing problematic samples during fixed-point solving, we skip them to ensure training stability.
During training, we set the learning rate to $1\times10^{-4}$ and utilize stochastic gradient descent with a momentum of $0.9$ and weight decay of $5\times 10^{-4}$. Both $E_W$ and $E_H$ are set to $0.01$. The training phase for each network consists of 100 epochs, with a batch size of $128$. In this context, accuracy is assessed by averaging the results from the last 10 epochs and computing their standard deviation.

\subsection{Performance under balanced conditions}
\label{5.2}
\begin{wrapfigure}[8]{r}{0.55\textwidth}\vspace{1pt}
\begin{minipage}{0.55\textwidth}\vspace{-20pt}
\begin{table}[H]
\renewcommand\arraystretch{1.2}
\caption{Comparison of accuracy under balanced settings of Cifar-10 and Cifar-100}
\centering
\label{tab_balance}
\begin{tabular}{c|cc}
\hline
            & Cifar-10       & Cifar-100      \\ \hline
Explicit NN & $93.05\pm0.17$ & $64.35\pm0.20$ \\ 
DEQ         & $93.23\pm0.13$ & $64.77\pm0.36$ \\ \hline
\end{tabular}
\end{table}
\end{minipage}
\end{wrapfigure}

By using the settings in (\ref{opt_exp}) and (\ref{DEQ_NC}), we compared the performance of DEQ and Explicit NN on Cifar-10 \cite{krizhevsky2010cifar} and Cifar-100 \cite{Krizhevsky09learningmultiple} for validation, as shown in Figure \ref{fig_balance}. Their $\mathcal{NC}$ performances remain comparable, i.e., DEQ achieves results similar to Explicit NN, corroborating the findings of Theorem \ref{theorem_nc_balance}. As for accuracy, from the results in the first column of Table \ref{tab_balance}, it can be observed that DEQ's accuracy is higher than that of the explicit layer, which aligns with Theorem \ref{thm_balance}. However, the increase is only marginal due to the fact that the coefficients $E_H$ and $E_W$ act as scaling factors. Therefore, compared to explicit neural network, DEQ finds it challenging to achieve a significantly lower loss and, consequently, a substantial improvement. Moreover, Explicit NN performs well in fitting balanced datasets, so the accuracy of DEQ does not experience a significant boost in this context.

Here, we manually set the number of epochs to $100$ to avoid potential instability issues with DEQ as training deepens. This is because DEQ can be challenging to reach the TPT (Terminal Phase of Training). As the number of parameters increases, achieving fixed-point convergence becomes more difficult, and even parameter explosion may occur. Under the current vanilla design, it is challenging to avoid such instability. Therefore, for a fair comparison, we apply the same training settings to both the implicit DEQ and the explicit neural network. The results in Figure \ref{fig} indicate that the test performance at 100 epochs is not significantly different from that at TPT. Since DEQ shares the same backbone as the corresponding explicit neural network, it can still demonstrate better $\mathcal{NC}$ performance after reaching TPT in these cases.

\begin{table*}[htbp]

\renewcommand\arraystretch{1.2}
\caption{Test Accuracy on Cifar-10 and Cifar-100 Dataset with $K_A=3$}
\vspace{5pt}
\centering
\resizebox{\textwidth}{!}{
\begin{tabular}{c|c|ccc|ccc}
\hline
                             &          & \multicolumn{3}{c|}{Cifar-10}                    & \multicolumn{3}{c}{Cifar-100}                    \\ \cline{2-8} 
                             & $R$        & 10             & 50             & 100            & 10             & 50             & 100            \\ \hline
\multirow{3}{*}{Explicit NN} & overall  & 72.57$\pm$0.25 & 44.32$\pm$0.23 & 32.14$\pm$0.81 & 41.41$\pm$0.56 & 28.18$\pm$0.42 & 23.43$\pm$0.92 \\
                             & majority & 96.40$\pm$0.32 & 96.80$\pm$0.29 & 91.67$\pm$0.61 & 73.03$\pm$0.62 & 74.53$\pm$0.55 & 73.46$\pm$0.56 \\
                             & minority & 62.36$\pm$0.12 & 21.83$\pm$0.20 & 6.64$\pm$0.99  & 27.86$\pm$0.39 & 8.31$\pm$0.38  & 1.99$\pm$1.06  \\ \hline
\multirow{3}{*}{DEQ}         & overall  & 73.84$\pm$0.72 & 46.08$\pm$1.06 & 34.18$\pm$1.28 & 43.72$\pm$0.60 & 30.46$\pm$1.27 & 24.78$\pm$1.93 \\
                             & majority & 96.68$\pm$0.87 & 96.63$\pm$0.98 & 93.33$\pm$1.36 & 74.16$\pm$0.82 & 73.63$\pm$0.95 & 74.89$\pm$0.88 \\
                             & minority & 64.06$\pm$0.66 & 24.42$\pm$1.32 & 8.83$\pm$1.08  & 30.67$\pm$0.53 & 11.96$\pm$1.66 & 3.31$\pm$2.45  \\ \hline
\end{tabular}
}
\label{tab_acc100}
\end{table*}

\begin{table*}[htbp]

\renewcommand\arraystretch{1.2}
\caption{Test Accuracy on Cifar-100 Dataset with $K_A=3$}
\vspace{5pt}
\centering
\resizebox{\textwidth}{!}{
\begin{tabular}{c|c|ccc}
\hline
                             &          & \multicolumn{3}{c}{Cifar-100}                    \\ \cline{2-5} 
                             & $R$        & 10             & 50             & 100            \\ \hline
\multirow{3}{*}{Explicit NN} & overall  & 41.41$\pm$0.56 & 28.18$\pm$0.42 & 23.43$\pm$0.92 \\
                             & majority & 73.03$\pm$0.62 & 74.53$\pm$0.55 & 73.46$\pm$0.56 \\
                             & minority & 27.86$\pm$0.39 & 8.31$\pm$0.38  & 1.99$\pm$1.06  \\ \hline
\multirow{3}{*}{DEQ}         & overall  & 43.72$\pm$0.60 & 30.46$\pm$1.27 & 24.78$\pm$1.93 \\
                             & majority & 74.16$\pm$0.82 & 73.63$\pm$0.95 & 74.89$\pm$0.88 \\
                             & minority & 30.67$\pm$0.53 & 11.96$\pm$1.66 & 3.31$\pm$2.45  \\ \hline
\end{tabular}
}
\label{tab_acc100}
\end{table*}

\subsection{Performance under imbalanced conditions}

We conducted experiments with varying configurations with different numbers of majority and minority classes and imbalance degrees. Assume the numbers of majority and minority classes are $(K_A,K_B)$ with corresponding sample sizes $(n_A,n_B)$, the imbalance degree is denoted as $R=n_A/n_B$.We considered different setups for majority and minority class quantities, such as $(3,7)$, $(5,5)$, and $(7,3)$. Additionally, we varied the ratio of sample quantities $R$ between majority and minority classes with values of $10$, $50$ and $100$. We also tested the phenomenon of $\mathcal{NC}$ and accuracy on the Cifar-10 and Cifar-100 datasets, which own a total of $5000$ images for each class. Specifically, when $R=100$ and $(K_A,K_B)=(3,7)$ for Cifar-10, the number of samples for all classes is $(5000,5000,5000,50,50,50,50,50,50,50)$.

The results for $(K_A,K_B)=(3,7)$ are shown in Table \ref{tab_acc100}, where the test dataset owns the same distribution as the training dataset. We use ``overall", ``majority", and ``minority" to represent the results across all categories, the majority class, and the minority class, respectively.
We contrasted the difference in the training outcomes between the Explicit Neural Network and DEQ, and the superior performance of DEQ compared to Explicit Neural Network confirms DEQ's higher learning potential. This suggests that DEQ can achieve a lower bound on its loss function. The experimental results indicate that DEQ consistently outperforms explicit neural network in accuracy during imbalanced training, aligning with Theorem \ref{thm_imbalance}. Specifically, we present the outcomes for $(K_A,K_B)=(5,5)$ with $R=100$ are depicted in Figure \ref{fig_imbalance}. The results strongly corroborate Theorem \ref{thm_imbalance}, affirming DEQ exhibits the same $\mathcal{NC}1$ phenomenon as an explicit neural network under these conditions. However, DEQ outperforms the explicit neural network in terms of $\mathcal{NC}2$ and $\mathcal{NC}3$. Additional experimental results with different parameters are detailed in Appendix Section \ref{supp_experiment}.

In addition to the stability considerations discussed in Section \ref{5.2}, we refrain from training for an extensive number of epochs due to the imbalance in the samples of the training set. This is because excessive learning rounds might cause the network parameters to predominantly capture information from the majority class, resulting in overfitting its features. This, in turn, diminishes the generalization of learning features from other classes, leading to marginal improvements in accuracy on the test set. As depicted in Figure \ref{fig_imbalance}, the model has already converged at this point. Moreover, limiting the number of training epochs helps to avoid the gradual instability in the learning process of DEQ.

\section{Conclusion}

{\color{black}
In this study, we have systematically analyzed the representation of Deep Equilibrium Models (DEQ) and explicit neural networks under both balanced and imbalanced conditions using the phenomenon of Neural Collapse ($\mathcal{NC}$). Our theoretical analysis demonstrated that $\mathcal{NC}$ is present in DEQ under balanced conditions. Furthermore, in imbalanced settings, DEQ exhibited notable advantages over explicit neural networks, such as the convergence of extracted features to the vertices of a simplex equiangular tight frame and self-duality properties under mild conditions. These findings highlight the superior performance of DEQ in handling imbalanced datasets. Our experimental results in both balanced and imbalanced scenarios validate the theoretical insights. The current analysis is limited to simple imbalanced scenarios and the linear structure of DEQ models. Future work will expand on this foundation by exploring more general imbalanced scenarios and extending the analysis to more complex forms of DEQ models. 

\section*{Acknowledgement}
This work was supported by NSFC (No.62303319), Shanghai Sailing Program (22YF1428800), Shanghai Local College Capacity Building Program (23010503100), ShanghaiTech AI4S Initiative SHTAI4S202404, Shanghai Frontiers Science Center of Human-centered Artificial Intelligence (ShangHAI), MoE Key Laboratory of Intelligent Perception and Human-Machine Collaboration (ShanghaiTech University) and Shanghai Engineering Research Center of Intelligent Vision and Imaging.

\bibliographystyle{plain}
\bibliography{final}

\newpage

\appendix
\addtocontents{toc}{\protect\setcounter{tocdepth}{3}}
\renewcommand{\contentsname}{Appendix Contents}

\tableofcontents  

\section{Evaluation metrics of \texorpdfstring{$\mathcal{NC}$}{NC}} 

Followed by the settings of \cite{tirer2022extended} and \cite{dang2023neural}, the measurement of $\mathcal{NC}$ are set as follow:

Let $\boldsymbol h_k\triangleq \frac{1}{n_k}\sum_{i=1}^{n_k} \boldsymbol h_{k,i}$ represent the average of all features within class $k$ and these $K$ classes collectively constitute the average matrix $\bar{\boldsymbol H}=\left[\boldsymbol h_1,\cdots,\boldsymbol h_K\right]$. Besides, The global average is defined as $\boldsymbol h_G\triangleq \frac{1}{K}\sum_{i=1}^{K} \boldsymbol h_{k}$. Subsequently, the within-class and between-class covariances can be calculated as:
\begin{equation}
    \begin{aligned}
        \boldsymbol \Sigma_W&\triangleq\frac{1}{N}\sum\limits_{k=1}^K\sum\limits_{i=1}^{n_k}(\boldsymbol h_{k,i}-\boldsymbol h_k)(\boldsymbol h_{k,i}-\boldsymbol h_k)^T,\\
        \boldsymbol \Sigma_B&\triangleq\frac{1}{K}\sum\limits_{k=1}^K(\boldsymbol h_k-\bar{\boldsymbol h}_G)(\boldsymbol h_k-\bar{\boldsymbol h}_G)^T.
    \end{aligned}
\end{equation}

$\mathcal{NC}1$ measures the variation of features with-in the same class:
\begin{equation}
    \mathcal{NC}1=\frac{1}{K}\text{tr}\left(\boldsymbol \Sigma_W\boldsymbol \Sigma_B^\dagger\right),
\end{equation}
where $\boldsymbol \Sigma_B^\dagger$ denotes the pseudo-inverse of $\boldsymbol \Sigma_B$.

$\mathcal{NC}2$ measures similarity between the mean of learned last-layer features $\bar{\boldsymbol H}$ and the structure of Simplex ETF:

\begin{equation}
    \mathcal{NC}2=\left\|\frac{\bar{\boldsymbol H}^T\bar{\boldsymbol H}}{\|\bar{\boldsymbol H}^T\bar{\boldsymbol H}\|_F}-\frac{1}{K-1}(\boldsymbol I_K-\frac{1}K \boldsymbol 1_K\boldsymbol 1_K^T)\right\|_F.
\end{equation}

$\mathcal{NC}3$ measures similarity of the last-layer feature $\bar{\boldsymbol H}$ and weights of classifier $\boldsymbol W$:
\begin{equation}
    \mathcal{NC}3=\left\|\frac{\boldsymbol W}{\|\boldsymbol W\|_F}-\frac{\bar{\boldsymbol H}}{\|\bar{\boldsymbol H}\|_F}\right\|.
\end{equation}
Additionally, it is worth noting that all above $\mathcal{NC}$ criteria are exclusively based on the training set. This is because our focus is solely on analyzing learning performance on imbalanced datasets, and generalization is not a primary concern.

\section{Proof under balanced setting}
\label{Supp_sec A}
\subsection{Problem definition}
As different layers in the neural network introduce complexity, the optimization problem is non-convex, and KKT conditions do not guarantee global optimality. Therefore, we consider applying inequality relaxation to the joint optimization problem, obtaining a lower bound for the loss function. By determining the conditions under which the equality holds, we can derive the requirements for the $\mathcal{NC}$ phenomenon. This analysis assumes a balanced setting, where all $\#\pi(k) = n_1 = n_2 = \cdots = n_K = N/K$.

We considered the fully connected layers (explicit) and Deep Equilibrium Models (implicit) under the balanced settings respectively, and then derived the detailed proof.

\begin{figure}[htbp]
  \begin{minipage}{0.5\textwidth}
  \textbf{(Fully Connected Layers)}
    \begin{equation*}
        \begin{aligned}
            \min\limits_{\boldsymbol W, \boldsymbol W_\text{EX}, \boldsymbol H} ~~~~&\frac{1}N\sum\limits_{k=1}^K\sum\limits_{i=1}^{n_k} \mathcal{L}({\boldsymbol W} \boldsymbol W_\text{EX}\boldsymbol h^0_{k,i},{\boldsymbol y}_k)\\
            \text{s.t.}~~~~&\boldsymbol h_{k,i}=\boldsymbol W_\text{EX}\boldsymbol h^0_{k,i},\\
            &\|\boldsymbol W_\text{EX}\|_F\leq E_H,\\
            &\frac{1}K \sum\limits_{k=1}^K \|\boldsymbol w_k\|^2\leq E_W,\\
            &\frac{1}K \sum\limits_{k=1}^K \frac{1}{n_k} \sum\limits_{i=1}^{n_k} \| \boldsymbol h_{k,i}\|^2\leq E_H,
        \end{aligned}
    \end{equation*}
  \end{minipage}%
  \begin{minipage}{0.5\textwidth}
    \textbf{(Deep Equilibrium Models)}
    {\begin{equation*}
        \begin{aligned}
            \min\limits_{\boldsymbol W,\boldsymbol W_\text{DEQ}, \boldsymbol z_{k,i}^\star} ~~~~&\frac{1}N\sum\limits_{k=1}^K\sum\limits_{i=1}^{n_k} \mathcal{L}({\boldsymbol W}\boldsymbol z^\star,{\boldsymbol y}_k)\\
            \text{s.t.}~~~~&\boldsymbol z_{k,i}^\star=f(\boldsymbol h^0_{k,i}; \boldsymbol W_\text{DEQ}),\\
            & \|\boldsymbol W_\text{DEQ}\|_F\leq E_H,\\
            &\frac{1}K \sum\limits_{k=1}^K \|\boldsymbol w_k\|^2\leq E_W,\\
            &\frac{1}K \sum\limits_{k=1}^K \frac{1}{n_k} \sum\limits_{i=1}^{n_k} \| \boldsymbol z^\star_{k,i}\|^2\leq E_H.
        \end{aligned}
    \end{equation*}}
  \end{minipage}
   Note that here $n_1=n_2=\cdots=n_k=n$, and $f$ represents the form of Linear DEQ, where we will use $f(\boldsymbol x;\boldsymbol W_\text{DEQ})=\sum\limits_{i=1}^{\infty}\boldsymbol W^i_\text{DEQ}\boldsymbol x$ for representation in the following proofs.
\end{figure}

In a classification task, cross-entropy loss $\mathcal{L}({\boldsymbol W}\boldsymbol h_{k,i},{\boldsymbol y}_k)$ is regarded as the final loss function. Drawing inspiration from \cite{fang2021exploring}, our initial efforts revolve around organizing and simplifying the log function to distinguish the logit in class $k$ from other classes.

First consider the following lemma:
\begin{lemma}
\label{lemma_log}
Let there be $K$ variables $\delta_1, \delta_2, \cdots, \delta_K$, and the logit of each variable $\delta_k$ satisfies the inequality:
\begin{equation}
    \log\left(\delta_k/\sum\limits_{k=1}^k\delta_k\right)\leq M_1\left(\log\delta_k-\frac{1}{K-1}\sum\limits_{k'\neq k}^K\log\delta_k\right)+M_2,
\end{equation}
where $M_1$ and $M_2$ are predefined constants.
\end{lemma}
\begin{proof}
Split the sum in the denominator and sequentially introduce weights for each term. Here, define $K$ coefficients such that their sum is 1. Therefore, we have:
\begin{equation}
    \frac{C_1}{C_1+C_2}+\underbrace{C_3+\cdots +C_3}_{K-1} = 1,
\end{equation}
that is $C_3 = \dfrac{C_2}{(K-1)(C_1+C_2)}$. Therefore, by Jensen's inequality, we can derive:
\begin{equation}
\label{ineq_1}
\begin{aligned}
    \log\left(\delta_k/\sum\limits_{k'=1}^{K}\delta_{k'}\right)&=\log \delta_k - \log\left(\sum\limits_{k'=1}^K\delta_{k'}\right)\\
    &=\log \delta_k-\log\left(\frac{C_1}{C_1+C_2}\frac{(C_1+C_2)\delta_k}{C_1}+C_3\sum\limits_{k'\neq k}^K\frac{\delta_{k'}}{C_3}\right)\\
    &\leq \log \delta_k-\frac{C_1}{C_1+C_2}\log\left(\frac{(C_1+C_2)\delta_k}{C_1}\right)-C_3\sum\limits_{k'\neq k}^K\log\frac{\delta_{k'}}{C_3}\\
    &=M_1\left(\log\delta_k-\frac{1}{K-1}\sum\limits_{k'\neq k}^K\log\delta_{k'}\right)+M_2,
\end{aligned}
\end{equation}
where $M_1=\dfrac{C_2}{C_1+C_2}$ and $M_2=\dfrac{C_2}{C_1+C_2}\log C_3-\dfrac{C_1}{C_1+C_2}\log\left(\dfrac{C_1+C_2}{C_1}\right)$. Therefore the lemma is proved.
\end{proof}

\begin{remark}
\label{remark1}
When $C_2/C_1=\dfrac{1}{K-1}\exp{\left(\log\delta_k-\dfrac{1}{K-1}\sum\limits_{k'\neq k}^K\log\delta_{k'}\right)}$, the right term of the inequality in lemma \ref{lemma_log} reaches its maximum.
\end{remark}
\begin{proof}
    Let $\mathcal{M} = M_1 \left(\log \delta_k - \dfrac{1}{K-1} \sum\limits_{k' \neq k}^K \log \delta_{k'}\right) + M_2$ in Lemma \ref{lemma_log}. Then, upon computing the derivatives of $C_1$ and $C_2$, we obtain:
    \begin{equation}
    \begin{aligned}
        \frac{\partial \mathcal{M}}{\partial C_1}&=\frac{1}{(C_1+C_2)^2}\left(-C_2\frac{\mathcal{M}-M_2}{M_1}+C_2\log\frac{(K-1)C_2}{C_1}\right)\\
        \frac{\partial \mathcal{M}}{\partial C_2}&=\frac{1}{(C_1+C_2)^2}\left(C_1\frac{\mathcal{M}-M_2}{M_1}-C_1\log\frac{(K-1)C_2}{C_1}\right).
    \end{aligned}
    \end{equation}
    Combining these two equations yields the conclusion.
\end{proof}

Next, we substitute the result of each logit into the lemma \ref{lemma_log}, from which we can derive:
\begin{equation}
\label{loss1}
\begin{aligned}
    \mathcal{L} &= -\frac{1}N\sum_{k=1}^K \sum_{i=1}^{n} \boldsymbol y_{k,i} \log\frac{ \exp(\boldsymbol w^T_{k,i} \boldsymbol h_i)} {\sum\limits_{k'=1}^K \exp(\boldsymbol w^T_{k',i} \boldsymbol h_i)}\\
    &\geq\frac{C_1}{(C_1+C_2)N(K-1)}\sum\limits_{i=1}^n\left[\left(\sum_{k=1}^K\boldsymbol h_{k,i}\right)^T\left(\sum_{k=1}^K\boldsymbol w_{k}\right)-K\sum_{k=1}^K \boldsymbol h^T_{k,i}\boldsymbol w_k\right]+C_4\\
    &=\frac{C_1K}{(C_1+C_2)N(K-1)}\sum\limits_{k=1}^K\sum\limits_{i=1}^{n}(\bar{\boldsymbol h}_i-\boldsymbol h_{k,i})^T\boldsymbol w_k+C_4\\
    &\geq\frac{C_1}{(C_1+C_2)N(K-1)}\left(-\frac{K}2\sum\limits_{k=1}^K\sum\limits_{i=1}^{n}\|\bar{\boldsymbol h}_i-\boldsymbol h_{k,i}\|^2/C_5-\frac{C_5N}2\sum\limits_{k=1}^K\|\boldsymbol w_k\|^2\right)+C_4,
\end{aligned}
\end{equation}
where the second inequality applies Mean Inequalities, and $C_4=\dfrac{C_2}{N(C_1+C_2)}\log C_3-\dfrac{C_1}{N(C_1+C_2)}\log\left(\dfrac{C_1+C_2}{C_1}\right)$. For convenience, we denote $\tilde{\mathcal{L}}=-\dfrac{K}2\sum\limits_{k=1}^K\sum\limits_{i=1}^{n}\|\bar{\boldsymbol h}_i-\boldsymbol h_{k,i}\|^2/C_5-\dfrac{C_5N}2\sum\limits_{k=1}^K\|\boldsymbol w_k\|^2$. As the corresponding constraints have already been added in (\ref{opt_prob}), specifically the constraint $\sum\limits_{k=1}^K\|\boldsymbol w_k\|^2\leq E_W$, our focus shifts to discussing the situation concerning the first term. Since it represents the features of the final layer, we separately explore the differences in its extraction when using DEQ and fully connected layers. First suppose the extracted feature by the backbone is $\boldsymbol h^0$.

\subsection{\texorpdfstring{$\mathcal{NC}$}{NC} analysis}
We separately discuss the representation of $\mathcal{NC}$ in the cases of Explicit NN and DEQ, and compare the lower bounds of the loss function.

\subsubsection{\texorpdfstring{$\mathcal{NC}$}{NC} proof in Explicit neural networks}
For convenience, we assume there is only one layer in the feature extractor, that is, $\boldsymbol h = \boldsymbol W_\text{EX} \boldsymbol h^0$, then the first term in $\tilde{\mathcal{L}}$ becomes:
\begin{equation}
\label{ineq2}
    \begin{aligned}
        -\frac{K}2\sum\limits_{k=1}^K\sum\limits_{i=1}^{n}\|\bar{\boldsymbol h}_i-\boldsymbol h_{k,i}\|^2 &=-\frac{K}2\sum\limits_{k=1}^K\sum\limits_{i=1}^{n}\|\boldsymbol W_\text{EX} (\bar{\boldsymbol h}^0_i-\boldsymbol h^0_{k,i})\|^2\\
        &\geq -\frac{K}4\sum\limits_{k=1}^K\sum\limits_{i=1}^{n}\left(\|\boldsymbol W_\text{EX}\|_F^2+\|\bar{\boldsymbol h}^0_i-\boldsymbol h^0_{k,i}\|^2\right).\\
    \end{aligned}
\end{equation}
Substituting them into the loss function (\ref{loss1}), we can observe that:
\begin{equation}
\label{loss2}
    \begin{aligned}
        \tilde{\mathcal{L}} &\geq -\frac{NK}{4C_5}\|\boldsymbol W_\text{EX}\|_F^2 - \frac{K}{4C_5}\sum\limits_{k=1}^K\sum\limits_{i=1}^{n}\|\bar{\boldsymbol h}^0_i-\boldsymbol h^0_{k,i}\|^2-\frac{C_5NK}2E_W\\
        &=-\frac{K^2}{4C_5}\sum\limits_{i=1}^{n}\frac{1}K\sum\limits_{k=1}^K\left( \|\boldsymbol h^0_{k,i}\|^2-\|\bar{\boldsymbol h}^0_i\|^2\right)
        -\frac{NK}{4C_5}\|\boldsymbol W_\text{EX}\|_F^2-\frac{C_5NK}2E_W\\
        &\geq -\frac{KN}{4C_5} E_H - \frac{C_5KN}2E_W+\frac{K^2}{4C_5}\sum\limits_{i=1}^{n}\|\bar{\boldsymbol h}_i^0\|^2-\frac{NK}{4C_5}E_H.
    \end{aligned}
\end{equation}
To acquire the lower bound of the loss function, we assign the value $C_5=\sqrt{E_H/E_W}$, the lower bound becomes: 
\begin{equation}
\label{loss_fc}
    \inf \mathcal{L}_\text{EX} = -\frac{C_1K}{(C_1+C_2)(K-1)}\sqrt{E_WE_H}+C_4.
\end{equation}

Furthermore, the condition $\|\bar{\boldsymbol h}_i^0\|^2=0$ should also be satisfied, indicating that the average of the features for the $i$-th sample, $\dfrac{1}{K}\sum\limits_{k=1}^K \boldsymbol h^0_{k,i}$, is equal to zero.

The satisfaction conditions for the inequalities include the following: 
\begin{itemize}
    \item In Eq. (\ref{loss1}):
    The first inequality becomes equality when
    \begin{equation}
        \frac{(C_1+C_2) \boldsymbol h^T_{k,i}\boldsymbol w_k}{C_1}=\frac{\boldsymbol h^T_{k,i}\boldsymbol w_{k'}}{C_3},
    \end{equation}
    that is, 
    \begin{equation}
         h^T_{k,i}\boldsymbol w_k =  h^T_{k,i}\boldsymbol w_{k'}+\log\left(\frac{C_1(K-1)}{C_2}\right).    
    \end{equation}
    The second inequality is reduced to equality when $\bar{\boldsymbol h}_i-\boldsymbol h_{k,i}=-C_5\boldsymbol w_k$.
    \item In Eq. (\ref{ineq2}): $\|\boldsymbol W_\text{EX}\|_F^2=\sum\limits_{k=1}^K \sum\limits_{i=1}^{n_k}\|\bar{\boldsymbol h}^0_i-\boldsymbol h^0_{k,i}\|^2.$
    \item In Eq. (\ref{loss2}):
    When the following condition $\dfrac{1}K \sum\limits_{k=1}^K \|\boldsymbol w_k\|^2= E_W$ and $\|\boldsymbol W_\text{EX}\|_F^2=\dfrac{1}K \sum\limits_{k=1}^K \sum\limits_{i=1}^{n_k} \| \boldsymbol h^0_{k,i}\|^2 = E_H$ holds, the inequality was reduced to equality.
\end{itemize}

Since $\|\bar{\boldsymbol h}_i^0\|^2=0$, it follows that $\|\bar{\boldsymbol h}_i\|^2=\|\boldsymbol W_\text{EX}\bar{\boldsymbol h}_i^0\|^2=0$. Combined with the condition $\dfrac{1}K \sum\limits_{k=1}^K \|\boldsymbol w_k\|^2= E_W$ and $\dfrac{1}K \sum\limits_{k=1}^K \sum\limits_{i=1}^{n_k} \| \boldsymbol h_{k,i}\|^2 = E_H$, therefore, $\boldsymbol h_k=\boldsymbol h_{k,i}$, for $\forall k$, that is, $\mathcal{NC}1$ is proved.

Consequently, $\boldsymbol h_{k,i}=C_5\boldsymbol w_k$, demonstrating the validity of $\mathcal{NC}3$.

For $\mathcal{NC}2$, since
\begin{equation}
{\fontsize{8.5}{10}
    \begin{aligned}
        \sqrt{E_H/E_W}\|\boldsymbol w_k\|^2&=\boldsymbol h_k \boldsymbol w_k=\boldsymbol h_k \boldsymbol w_{k'}+\log\left(\frac{C_1(K-1)}{C_2}\right)=\boldsymbol W_\text{EX}\boldsymbol h_k^0\boldsymbol W_{k'}+\log\left(\frac{C_1(K-1)}{C_2}\right),\\
        \sqrt{E_H/E_W}\|\boldsymbol w_{k'}\|^2&=\boldsymbol h_{k'} \boldsymbol w_{k'}=\boldsymbol h_{k'} \boldsymbol w_k+\log\left(\frac{C_1(K-1)}{C_2}\right)=\boldsymbol W_\text{EX}\boldsymbol h_{k'}^0\boldsymbol W_k+\log\left(\frac{C_1(K-1)}{C_2}\right)
    \end{aligned}
    }
\end{equation}
holds, by the equality conditions, $\|\boldsymbol w_k\|^2=\|\boldsymbol w_{k'}\|^2=E_W$.

Further, $\sum\limits_{k=1}^K \boldsymbol h_k\boldsymbol w_{k'}=\sum\limits_{k=1}^K \boldsymbol W_\text{EX}\boldsymbol h^0_k\boldsymbol w_{k'}=0$, as $\boldsymbol h_k\boldsymbol w_k=\sqrt{E_WE_H}$, so $\boldsymbol h_k\boldsymbol w_{k'}=-\dfrac{\sqrt{E_WE_H}}{N-1}$.

Therefore, the $\mathcal{NC}2$ condition satisfies:
\begin{equation}
    \boldsymbol W\boldsymbol W^T = \sqrt{E_W/E_H}\boldsymbol W\boldsymbol H = \frac{KE_W}{K-1}\left(\boldsymbol 1_K -\frac{1}k\boldsymbol 1_K\boldsymbol 1_K^T\right).
\end{equation}

\subsubsection{\texorpdfstring{$\mathcal{NC}$}{NC} proof in DEQ}

In the blocks for feature extraction, DEQ can be referred as a mapping from the features by backbone to the output $\boldsymbol h^0\rightarrow \boldsymbol h^\star$, which can be directly solved using the implicit equation:
\begin{equation}
    \boldsymbol h^\star = f(\boldsymbol W_\text{DEQ}; \boldsymbol h^0)=\sum\limits_{i=1}^{\infty}\boldsymbol W^i_\text{DEQ} \boldsymbol h^0.
\end{equation}
Similar as the explicit case, start with the term:
\begin{equation}
    -\frac{K}2\sum\limits_{k=1}^K\sum\limits_{i=1}^{n}\|\bar{\boldsymbol h}_i-\boldsymbol h_{k,i}\|^2 =\frac{K}2\sum\limits_{k=1}^K\sum\limits_{i=1}^{n}\left\|\sum\limits_{j=0}^\infty\boldsymbol W^j_\text{DEQ} (\bar{\boldsymbol h}^0_i-\boldsymbol h^0_{k,i})\right\|^2.
\end{equation}
Since the Neumann series can be regarded as a recursive procedure, denote $\mathcal{G}_{k,i}^j=\sum\limits_{j'=0}^j\boldsymbol W^{j'}_\text{DEQ} (\bar{\boldsymbol h}^0_i-\boldsymbol h^0_{k,i})$ ($j=0,1,\cdots,\infty$), therefore $\mathcal{G}_{k,i}^j=\boldsymbol W_\text{DEQ}\mathcal{G}_{k,i}^{j-1}+(\bar{\boldsymbol h}^0_i-\boldsymbol h^0_{k,i})$.

\begin{equation}
{\fontsize{9}{10}
\label{ineq3}
    \begin{aligned}
    -\frac{K}2\sum\limits_{k=1}^K\sum\limits_{i=1}^{n}\left\|\mathcal{G}_{k,i}^j\right\|^2 &= \frac{K}2\sum\limits_{k=1}^K\sum\limits_{i=1}^{n}\left\|\boldsymbol W_\text{DEQ}\mathcal{G}_{k,i}^{j-1}+(\bar{\boldsymbol h}^0_i-\boldsymbol h^0_{k,i})\right\|^2\\
    &\geq-\frac{K}2\sum\limits_{k=1}^K\sum\limits_{i=1}^{n}\left\|\boldsymbol W_\text{DEQ}\mathcal{G}_{k,i}^{j-1}\right\|^2-\frac{K}2\sum\limits_{k=1}^K\sum\limits_{i=1}^{n}\left\|\bar{\boldsymbol h}^0_i-\boldsymbol h^0_{k,i}\right\|^2\\
    &\geq-\frac{K}4\sum\limits_{k=1}^K\sum\limits_{i=1}^{n}\left(\left\|\boldsymbol W_\text{DEQ}\right\|_F^2+\left\|\mathcal{G}_{k,i}^{j-1}\right\|^2\right)-\frac{K}2\sum\limits_{k=1}^K\sum\limits_{i=1}^{n}\left\|\bar{\boldsymbol h}^0_i-\boldsymbol h^0_{k,i}\right\|_F^2.\\    
    \end{aligned}
    }
\end{equation}

Continuing the recursion, we can obtain:
\begin{equation}
    -\frac{1}2\|\mathcal{G}^j_{k,i}\|^2\geq -\left(\frac{1}2\right)^{j+1}\left\|\mathcal{G}^0_{k,i}\right\|^2-\left(1-\frac{1}{2^j}\right)\left\|{\boldsymbol h}^0_i-\boldsymbol h^0_{k,i}\right\|^2-\left(\frac{1}2-\frac{1}{2^{j+1}}\right)\left\|{\boldsymbol W}_\text{DEQ}\right\|_F^2.
\end{equation}
So, when $j\rightarrow\infty$,
\begin{equation}
    \begin{aligned}
    -\frac{K}2\sum\limits_{k=1}^K\sum\limits_{i=1}^{n}\left\|\mathcal{G}_{k,i}^0\right\|^2&=\frac{K}2\sum\limits_{k=1}^K\sum\limits_{i=1}^{n}\left\|\sum\limits_{j=0}^\infty\boldsymbol W^j_\text{DEQ} (\bar{\boldsymbol h}^0_i-\boldsymbol h^0_{k,i})\right\|^2\\
    &\geq-K\sum\limits_{k=1}^K\sum\limits_{i=1}^{n}\left(\left\|{\boldsymbol h}^0_i-\boldsymbol h^0_{k,i}\right\|^2-\frac{1}2\left\|{\boldsymbol W}_\text{DEQ}\right\|_F^2\right).
    \end{aligned}
\end{equation}
Therefore, use a similar proof as a fully connected layer,
\begin{equation}
\label{loss3}
\begin{aligned}
    \tilde{\mathcal{L}} &\geq -\frac{K}{C_5}\sum\limits_{k=1}^K\sum\limits_{i=1}^{n}\left\|{\boldsymbol h}^0_i-\boldsymbol h^0_{k,i}\right\|^2-\frac{NK}{2C_5}\left\|{\boldsymbol W}_\text{DEQ}\right\|_F^2-\frac{C_5NK}2E_W\\
    &=-\frac{K}{C_5}\sum\limits_{i=1}^{n}\left(\frac{1}{K^2}\sum\limits_{k=1}^K\|\boldsymbol h^0_{k,i}\|^2-\|\bar{\boldsymbol h}^0_i\|^2\right)-\frac{NK}{2C_5}\left\|{\boldsymbol W}_\text{DEQ}\right\|_F^2-\frac{C_5NK}2E_W\\
    &\geq -\frac{NK}{C_5E_H}-\frac{C_5NK}2E_W+\frac{K^2}{C_5}\sum\limits_{i=1}^{n}\|\bar{\boldsymbol h}^0_i\|^2-\frac{NK}{2C_5}E_H.
\end{aligned}
\end{equation}
Set $C_5=\sqrt{E_H/E_W}$, the loss bound of the loss function becomes:

\begin{equation}
    \inf \mathcal{L}_\text{DEQ} = -\frac{2C_1K}{(C_1+C_2)(K-1)}\sqrt{E_WE_H}+C_4.
\end{equation}
In comparison with the lower bound of the loss function (\ref{loss_fc}), it is evident that the loss function of the DEQ layer is significantly lower than that of the explicit neural network. Since the models are identical, according to Remark \ref{remark1}, the values of $C_1$ and $C_2$ are nearly the same. This observation highlights the relatively stronger potential of DEQ compared to Explicit Neural Networks.

Also, the satisfaction conditions for the inequalities in DEQ settings include the following: 

\begin{itemize}
    \item In Eq. (\ref{loss1}):
    The first inequality becomes equality when
    \begin{equation}
        \frac{(C_1+C_2) \boldsymbol h^T_{k,i}\boldsymbol w_k}{C_1}=\frac{\boldsymbol h^T_{k,i}\boldsymbol w_{k'}}{C_3},
    \end{equation}
    that is, 
    \begin{equation}
         h^T_{k,i}\boldsymbol w_k =  h^T_{k,i}\boldsymbol w_{k'}+\log\left(\frac{C_1(K-1)}{C_2}\right).    
    \end{equation}
    The second inequality is reduced to equality when $\bar{\boldsymbol h}_i-\boldsymbol h_{k,i}=-C_5\boldsymbol w_k$.
    This condition is quite similar to explicit fully connected layers.

    \item In Eq. (\ref{ineq3}):
    
    The first inequality:
    \begin{equation}
        \boldsymbol W_\text{DEQ}\mathcal{G}_{k,i}^{j-1}=\bar{\boldsymbol h}^0_i-\boldsymbol h^0_{k,i},
    \end{equation} and the second inequality
    \begin{equation}
        \left\|\boldsymbol W_\text{DEQ}\right\|_F^2 = \left\|\mathcal{G}_{k,i}^{j-1}\right\|^2.
    \end{equation}

    \item In Eq. (\ref{loss3}):
    When the following condition $\dfrac{1}K \sum\limits_{k=1}^K \|\boldsymbol w_k\|^2= E_W$ and $\|\boldsymbol W_\text{DEQ}\|^2=\dfrac{1}K \sum\limits_{k=1}^K \sum\limits_{i=1}^{n_k} \| \boldsymbol h_{k,i}\|^2 = E_H$ holds, the inequality was reduced to equality.
\end{itemize}

To summarize, DEQs are proposed for the memory-saving properties, as the forward passes can leverage any black-box root solvers \cite{bai2019deep,bai2021stabilizing}.  However, in terms of forward inference, explicit neural networks have limited learning capacity for data representation since they involve direct expressions computed in a single pass and backward propagation. In contrast, DEQ, lacking a direct explicit form, requires multiple rounds of parameter adjustments for learning. In each iteration, DEQ introduces input data in a sequential manner, allowing more adjustment space for learning parameters specific to the input. Therefore, to compare the two loss functions, we can derive the following theorem:

\begin{theorem}
\label{thm_balance}
    DEQ achieves a lower bound on the loss function compared to explicit neural network under balanced datasets:
    \begin{equation*}
        \inf \mathcal{L}_\text{DEQ} = -2C_1\frac{K}{K-1}\sqrt{E_WE_H}+C_2,
    \end{equation*}
    while the lower bound of loss function of explicit neural network remains:
    \begin{equation*}
        \inf \mathcal{L}_\text{EX} = -C_1\frac{K}{K-1}\sqrt{E_WE_H}+C_2,
    \end{equation*}
    where $C_1$ and $C_2$ are two given constants.
\end{theorem}

Under the balanced dataset, the sample distribution of each class within each batch is relatively even. Therefore, during the fixed-point iteration process, both DEQ and explicit neural network can learn the features of each class relatively well, without showing significant differences. Besides, from a numerical perspective, the penalties $E_W$ and $E_H$ are generally not set to very large values, especially smaller than $1$, so the difference between the two lower bounds in Theorem \ref{thm_balance} may not be substantial. Besides, as analyzed in Remark \ref{remark1}, we can set $C_2$ in this two equations as identical, and once the propotion of logits in the explicit neural network is greater than the DEQ, the lower bound of loss function in DEQ is lower.

\section{Proof under imbalanced learning}
\label{supp_imb}
\subsection{Lower bound of the loss function}
Consider the loss function:
\begin{equation}
    \label{allloss}
    \mathcal{L}=\underbrace{\frac{K_An_A}N\sum\limits_{k=1}^{K_A}\sum\limits_{i=1}^{n_A} \mathcal{L}({\boldsymbol W}\boldsymbol h,{\boldsymbol y}_k)}_{\mathcal{L}_A}
    +\underbrace{\frac{K_Bn_B}N\sum\limits_{k=K_A+1}^{K_B}\sum\limits_{i=1}^{n_B} \mathcal{L}({\boldsymbol W}\boldsymbol h,{\boldsymbol y}_k)}_{\mathcal{L}_B}.
    \end{equation}

First analyze the loss in the majority class $\mathcal{L}_A$ and introduce each term in the loss function. Suppose sample $i$ belongs to category $k$, where $1\leq k\leq K_A$, i.e., $k$ is a majority class.

By applying Jensen's inequalities, we can derive:

\begin{equation}
{\fontsize{8}{5}
\begin{aligned}
    &\quad-\log \left(\frac{\exp (\boldsymbol h_{k,i}^T\boldsymbol w_k)}{\sum_{k'=1}^K \exp(\boldsymbol h_{k,i}^T\boldsymbol w_{k'})}\right)\\
    &=-\boldsymbol h_{k,i}^T\boldsymbol w_k+\log\left(C_1\exp\left(\frac{\boldsymbol h_{k,i}^T\boldsymbol w_k}{C_1}\right)+C_2\sum\limits_{k'\neq k}^{K_A}\exp\left(\frac{\boldsymbol h_{k,i}^T\boldsymbol w_{k'}}{C_2}\right)+C_3\sum\limits_{k'=K_A+1}^{K}\exp\left(\frac{\boldsymbol h_{k,i}^T\boldsymbol w_{k'}}{C_3}\right)\right)\\
    &\geq(C_1-1) \boldsymbol h_{k,i}^T\boldsymbol w_k+C_2\sum\limits_{k'\neq k}^{K_A}\boldsymbol h_{k,i}^T\boldsymbol w_{k'}+C_3\sum\limits_{k'=K_A+1}^{K}\boldsymbol h_{k,i}^T\boldsymbol w_{k'}+const\\
    &=C_0C_4\left(\frac{1}{K_A}\sum\limits_{k'=1}^{K_A}\boldsymbol h_{k,i}^T\boldsymbol w_{k'}-\boldsymbol h_{k,i}^T\boldsymbol w_k\right) + C_0C_5\left(\frac{1}{K_B}\sum\limits_{k'=K_A+1}^{K}\boldsymbol h_{k,i}^T\boldsymbol w_{k'}-\boldsymbol h_{k,i}^T\boldsymbol w_k\right)+const\\
    &=C_0C_4\left(\boldsymbol h_{k,i}^T\boldsymbol w_A-\boldsymbol h_{k,i}^T\boldsymbol w_k\right) + C_0C_5\left(\boldsymbol h_{k,i}^T\boldsymbol w_B-\boldsymbol h_{k,i}^T\boldsymbol w_k\right)+const.
\end{aligned}
}
\end{equation} 

Here the value of const is $-C_1\log C_1-(k_A-1)C_2\log C_2-K_BC_3\log C_3$. Besides, $\boldsymbol w_A=\dfrac{1}{K_A}\sum\limits_{k'=1}^{K_A}\boldsymbol w_{k'}$ and $\boldsymbol w_B=\dfrac{1}{K_B}\sum\limits_{k'=K_A+1}^{K}\boldsymbol w_{k'}$ represent the mean values of the weights in majority and minority classes, respectively.

To ensure the equality conditions hold, suppose there are three adaptive constants $a>0$, $b>0$, $c>0$. Denote $C_1=\dfrac{a}{a+(K_A-1)b+K_Bc}$, $C_2=\dfrac{b}{a+(K_A-1)b+K_Bc}$, and $C_3=\dfrac{c}{a+(K_A-1)b+K_Bc}$. 
Additionally, to ensure $C_4+C_5=1$, introduce a constant $C_0=\dfrac{K_Ab+K_Bc}{a+(K_A-1)b+K_Bc}$, thus $C_4=\dfrac{K_Ab}{K_Ab+K_Bc}$ and $C_5=\dfrac{K_Bc}{K_Ab+K_Bc}$.

After aggregating each term in the loss function, we obtain:

\begin{equation}
\label{loss_na}
\begin{aligned}
    &\quad\frac{1}{K_An_A}\sum\limits_{k=1}^{K_A}\sum\limits_{i=1}^{n_A} \mathcal{L}({\boldsymbol W}\boldsymbol h,{\boldsymbol y}_k)\\
    &\geq \frac{1}{K_An_A}\sum\limits_{k=1}^{K_A}\sum\limits_{i=1}^{n_A} C_4\left(\boldsymbol h_{k,i}^T\boldsymbol w_A-\boldsymbol h_{k,i}^T\boldsymbol w_k\right) + C_5\left(\boldsymbol h_{k,i}^T\boldsymbol w_B-\boldsymbol h_{k,i}^T\boldsymbol w_k\right)+const\\
    &=\frac{1}{K_A}\sum\limits_{k=1}^{K_A}\boldsymbol h_k^T(C_4\boldsymbol w_A+C_5\boldsymbol w_B-\boldsymbol w_k)+const,
\end{aligned}
\end{equation}
where $\boldsymbol h_{k}=\dfrac{1}{n_A}\sum\limits_{i=1}^{n_B}\boldsymbol  h_{k,i}$.

Subsequently, consider the lower bound of 
\begin{equation}
\label{ineq5}
\begin{aligned}
    \sum\limits_{k=1}^{K_A}\boldsymbol h_k^T(C_4\boldsymbol w_A+C_5\boldsymbol w_B-\boldsymbol w_k)\geq -\frac{C_6}2\sum\limits_{k=1}^{K_A}\|\boldsymbol h_k\|^2-\sum\limits_{k=1}^{K_A}\frac{1}2\|C_4\boldsymbol w_A+C_5\boldsymbol w_B-\boldsymbol w_k\|^2/C_6.
\end{aligned}
\end{equation}
Note that this inequality (\ref{ineq5}) is reduced to equality only when the following equality holds: 
\begin{equation}
\label{eq_major}
    C_4\boldsymbol w_A+C_5\boldsymbol w_B-\boldsymbol w_k = C_6\boldsymbol h_k,
\end{equation}
where $1\leq k\leq K_A$.

Continuing the analysis of inequality (\ref{ineq5}), the first term on the right-hand side can be bounded as:

\textbf{Case 1:} (Explicit fully connected layers)
\begin{equation}
\label{fc_minhk}
\begin{aligned}
    -\sum\limits_{k=1}^{K_A}\|\boldsymbol h_k\|^2&=-\sum\limits_{k=1}^{K_A}\|\boldsymbol W_\text{EX}\boldsymbol h_k^0\|^2\\
    &\geq -\frac{1}2\left(K_A\|\boldsymbol W_\text{EX}\|_F+\sum\limits_{k=1}^{K_A}\|\boldsymbol h_k^0\|^2\right)\\
    &\geq -\frac{1}2\left(K_A\|\boldsymbol W_\text{EX}\|_F+\sum\limits_{k=1}^{K_A}\frac{1}{n_k}\sum\limits_{i=1}^{n_k}\|\boldsymbol h_{k,i}^0\|^2\right)\\
    &\geq -K_AE_H.
\end{aligned}
\end{equation}

\textbf{Case 2:} (Deep Equilibrium Models)
\begin{equation}
\label{deq_minhk}
\begin{aligned}
    -\sum\limits_{k=1}^{K_A}\|\boldsymbol h_k\|^2&=-\sum\limits_{k=1}^{K_A}\left\|(\boldsymbol I - \boldsymbol W_\text{DEQ})^{-1}\boldsymbol h_k^0\right\|^2\\
    &\geq -\frac{1}2\left(K_A\sum\limits_{j=0}^{\infty}\left\|\boldsymbol W_\text{DEQ}\right\|^j_F+\sum\limits_{k=1}^{K_A}\|\boldsymbol h_k\|^2\right)\\
    &\geq -\frac{1}2\left(K_A\sum\limits_{j=0}^{\infty}E_H^j+\sum\limits_{k=1}^{K_A}\frac{1}{n_k}\sum\limits_{i=1}^{n_k}\|\boldsymbol h_{k,i}\|^2\right)\\
    &\geq -\frac{1}2\left(\frac{1}{1-E_H}+E_H\right).
\end{aligned}
\end{equation}

Compared the lower bound of explicit neural network and DEQ, we can find that:  \[\left(-\sum\limits_{k=1}^{K_A}\|\boldsymbol h_k\|^2\right)_\text{DEQ}<\left(-\sum\limits_{k=1}^{K_A}\|\boldsymbol h_k\|^2\right)_\text{EX}\] for all $E_H \neq 1$.

We now shift our attention to the second term (Ref. Eq [82-83] in \cite{fang2021exploring}):

\begin{equation}
\begin{aligned}
    &\quad-\frac{1}{K_A}\sum\limits_{k=1}^{K_A}\|C_4\boldsymbol w_A+C_5\boldsymbol w_B-\boldsymbol w_k\|^2\\
    &=-\frac{1}{K_A}\sum\limits_{k=1}^{K_A}\|\boldsymbol w_k\|^2+\frac{2}{K_A}\sum\limits_{k=1}^{K_A}\boldsymbol w_k^T(C_4\boldsymbol w_A+C_5\boldsymbol w_B)-\|C_4\boldsymbol w_A+C_5\boldsymbol w_B\|^2\\
    &=-\frac{1}{K_A}\sum\limits_{k=1}^{K_A}\|\boldsymbol w_k\|^2+2C_5^2\boldsymbol w_A^T\boldsymbol w_B + C_4(2-C_4)\|\boldsymbol w_A\|^2-C_5\|\boldsymbol w_B\|^2\\
    &=-\frac{1}{K_A}\sum\limits_{k=1}^{K_A}\|\boldsymbol w_k\|^2+\frac{1}{K_A}\sum\limits_{k=K_A+1}^{K}\|\boldsymbol w_k\|^2 + C_4(2-C_4)\left\|\boldsymbol w_A+\frac{C_5^2}{C_4(2-C_4)}\boldsymbol w_B\right\|^2 \\
    &\quad\quad\quad\quad\quad\quad\quad\quad\quad\quad\quad\quad\quad\quad- \left(C_5^2+\frac{C_5^2}{C_4(2-C_4)}\right)\|\boldsymbol w_B\|^2\\
    &\geq -\frac{K}{K_A}E_W + \left(\frac{1}{K_R}-C_5^2-\frac{C_5^4}{C_4(2-C_4)}\right)\|\boldsymbol w_B\|^2 + \frac{1}{K_A}\sum\limits_{k=K_A+1}^K\|\boldsymbol w_k-\boldsymbol w_B\|^2\\
    &\quad\quad\quad\quad\quad\quad\quad\quad\quad\quad\quad\quad\quad\quad +C_4(2-C_4)\left\|\boldsymbol w_A+\frac{C_5^2}{C_4(2-C_4)}\boldsymbol w_B\right\|^2,\\
\end{aligned}
\end{equation}
where $K_R = K_A/K_B$ denotes the ratio of the number of majority classes to minority classes.

In summary, the lower bound of loss function (\ref{loss_na}) could be simplified as:

\begin{equation}
{\fontsize{9}{5}
\label{ineq6}
\begin{aligned}
    \mathcal{L}_A&= \frac{1}{K_An_A}\sum\limits_{k=1}^{K_A}\sum\limits_{i=1}^{n_A} \mathcal{L}({\boldsymbol W}\boldsymbol h_{k,i},{\boldsymbol y}_k)\\
    &\geq\frac{1}{K_A}\sum\limits_{k=1}^{K_A}\boldsymbol h_k^T(C_4\boldsymbol w_A+C_5\boldsymbol w_B-\boldsymbol w_k)+const\\
    &\geq -\frac{C_6}{2K_A}\sum\limits_{k=1}^{K_A}\|\boldsymbol h_k\|^2-\frac{1}{2K_A}\sum\limits_{k=1}^{K_A}\|C_4\boldsymbol w_A+C_5\boldsymbol w_B-\boldsymbol w_k\|^2/C_6+const\\
    & \geq\frac{C_6}{2K_A} M -\frac{KE_W}{2C_6K_A} + \frac{1}{2C_6}\left(\frac{1}{K_R}-C_5^2-\frac{C_5^4}{C_4(2-C_4)}\right)\|\boldsymbol w_B\|^2\\
    &\quad\quad+\frac{C_4(2-C_4)}{C_6}\left\|\boldsymbol w_A+\frac{C_5^2}{C_4(2-C_4)}\boldsymbol w_B\right\|^2+ \frac{1}{2C_6K_A}\sum\limits_{k=K_A+1}^K\|\boldsymbol w_k-\boldsymbol w_B\|^2+const,
\end{aligned}
}
\end{equation}

where $M=-K_AE_H$ if the network is a fully connected layer and $M=-\frac{K_A}2\left(\frac{1}{1-E_H}+E_H\right)$ if the network is a Deep Equilibrium Model.

Similarly, the loss function w.r.t the minority classes is bounded as:
\begin{equation}
\label{ineq7}
\begin{aligned}
    \mathcal{L}_B&= \frac{1}{K_Bn_B}\sum\limits_{k=1}^{K_B}\sum\limits_{i=1}^{n_B} \mathcal{L}({\boldsymbol W}\boldsymbol h_{k,i},{\boldsymbol y}_k)\\
    &=\frac{1}{K_B}\sum\limits_{k=1}^{K_B}\boldsymbol h_k^T(C_4\boldsymbol w_A+C_5\boldsymbol w_B-\boldsymbol w_k)+const\\
    &\geq -\frac{C_6}{2K_B}\sum\limits_{k=1}^{K_B}\|\boldsymbol h_k\|^2-\frac{1}{2K_B}\sum\limits_{k=1}^{K_B}\|C_4\boldsymbol w_A+C_5\boldsymbol w_B-\boldsymbol w_k\|^2/C_6+const\\
    & \geq\frac{C_6}{2K_B} M -\frac{KE_W}{2C_6K_B} + \frac{1}{2C_6}\left({K_R}-C_5^2-\frac{C_5^4}{C_4(2-C_4)}\right)\|\boldsymbol w_A\|^2\\
    &\quad\quad+\frac{C_5(2-C_5)}{C_6}\left\|\frac{C_4^2}{C_5(2-C_5)}\boldsymbol w_A+\boldsymbol w_B\right\|^2+ \frac{1}{2C_6K_B}\sum\limits_{k=1}^{K_A}\|\boldsymbol w_k-\boldsymbol w_A\|^2+const.
\end{aligned}
\end{equation}

The inequality reduces to equality when the constraints in $\mathcal{C}$ are treated as equalities, achieving the upper bound. Additionally, the following equalities must hold:
\begin{equation}
\label{eq_minor}
    C_4\boldsymbol w_A+C_5\boldsymbol w_B-\boldsymbol w_k = C_6\boldsymbol h_k,
\end{equation}
where $K_A +1\leq k\leq K$.

If $K_R=1$, i.e., the number of majority classes is equal to the number of minority classes, the results of (\ref{ineq6}) and (\ref{ineq7}) are totally equivalent. 

Therefore, without loss of generality, assuming $K_A>K_B$, the lower bound of the loss function (\ref{allloss}) can be simplified to:

\begin{equation}
{\fontsize{8}{5}
\label{lower_loss}
    \begin{aligned}
        \mathcal{L}&=\mathcal{L}_A+\mathcal{L}_B\\
        &\geq \frac{C_6M}{2}\left(\frac{1}{K_A}+\frac{1}{K_B}\right)+ \frac{1}{2C_6K_B}\sum\limits_{k=1}^{K_A}\|\boldsymbol w_k-\boldsymbol w_A\|^2+\frac{1}{2C_6K_A}\sum\limits_{k=1}^{K_B}\|\boldsymbol w_k-\boldsymbol w_B\|^2\\
        &~~~+\frac{C_4(2-C_4)}{2C_6}\left\|\boldsymbol w_A+\frac{C_5^2}{C_4(2-C_4)}\boldsymbol w_B\right\|^2+\frac{C_5(2-C_5)}{2C_6}\left\|\frac{C_4^2}{C_5(2-C_5)}\boldsymbol w_A+\boldsymbol w_B\right\|^2\\
        &~~~+\frac{1}{2C_6}\left(\frac{1}{K_R}-C_5^2-\frac{C_5^4}{C_4(2-C_4)}\right)\|\boldsymbol w_B\|^2+\frac{1}{2C_6}\left({K_R}-C_5^2-\frac{C_5^4}{C_4(2-C_4)}\right)\|\boldsymbol w_A\|^2+const.
    \end{aligned}
}
\end{equation}

\subsection{\texorpdfstring{$\mathcal{NC}$}{NC} Analysis}

As analyzed in (\ref{fc_minhk}) and (\ref{deq_minhk}), when it reaches the minimal value, each $\boldsymbol h_{k,i}=\boldsymbol h_k$ for $\forall k=1,2\cdots,K_A$. Similarly, this holds for minority class with $K_A+1\leq k\leq K$. This implies that, in an imbalanced scenario, both DEQ and fully connected layer exhibit feature collapse, i.e., $\mathcal{NC}1$ is still present.

As we need to calculate the lower bound of the loss function, it is essential to minimize the terms $\mathcal{L}_A+\mathcal{L}_B$ as much as possible.

Therefore, consider the gradient with respect to $\boldsymbol w_k$ for majority class and $\boldsymbol w_k$ for minority class, respectively. First compute the case with $1\leq k\leq K_A$.

\begin{equation}
    \begin{aligned}
        \frac{\partial \mathcal{L}}{\partial \boldsymbol w_k}&=\frac{1}{C_6K_B}\left(1-\frac{1}{K_A}\right)(\boldsymbol w_k-\boldsymbol w_A)+\frac{C_4(2-C_4)}{C_6K_A}\left(\boldsymbol w_A+\frac{C_5^2}{C_4(2-C_4)}\boldsymbol w_B\right)\\
        &\quad+\frac{C_4^2}{K_AC_6}\left(\frac{C_4^2}{C_5(2-C_5)}\boldsymbol w_A+\boldsymbol w_B\right)\\
        &\quad\quad+\frac{1}{K_AC_6}\left({K_R}-C_5^2-\frac{C_5^4}{C_4(2-C_4)}\right)\boldsymbol w_A=0.
    \end{aligned}
\end{equation}

So, we can derive that
\begin{equation}
\label{condition_major}
\begin{aligned}
        \left(K_R-\frac{1}{K_B}\right)\boldsymbol w_k &+\frac{1}{1-C_4^2}\boldsymbol w_B\\ &+ \left(\frac{1}{K_B}+C_4(2-C_4)+\frac{C_4^4}{C_5(2-C_5)}-C_5^2-\frac{C_5^2}{C_4(2-C_4)}\right)\boldsymbol w_A=0.
\end{aligned}
\end{equation}

One important note here is that when the proportion of majority class samples approaches infinity, i.e., $C_4\rightarrow1$, we have $\frac{1}{1-C_4^2}\rightarrow0$. In this scenario, the weights $\boldsymbol w_k$ belonging to the majority class are almost exclusively related to $\boldsymbol w_A$, and have little dependence on the average of the minority class $\boldsymbol w_B$, which validates the results of Proposition \ref{prop_1}.

Similarly, if $K_A+1\leq k\leq K$, the following equality will hold to ensure optimality of $\boldsymbol w_k$ in minority classes of (\ref{lower_loss}):

\resizebox{\textwidth}{!}{
\begin{minipage}{\textwidth}
\begin{equation}
\label{condition_minor}
\begin{aligned}
    \left(\frac{1}{K_R}-\frac{1}{K_A}\right
    )\boldsymbol w_k &+ \frac{1}{1-C_5^2}\boldsymbol w_A\\ &+ \left(\frac{1}{K_A}+C_5(2-C_5)+\frac{C_5^4}{C_4(2-C_4)}-C_4^2-\frac{C_4^2}{C_5(2-C_5)}\right)\boldsymbol w_B=0. 
\end{aligned}
\end{equation}
\end{minipage}
}

Next, we consider the conditions for the validity in $\mathcal{NC}2$ and $\mathcal{NC}3$, then compare the performance of DEQ and explicit neural network. 

Therefore, for the majority class $1\leq k\leq K_A$, suppose it reaches its minimum value, recall the condition (\ref{eq_major}), and combined with (\ref{condition_major}), we can derive:

\begin{equation}
    C_6\boldsymbol h_k^T\boldsymbol h_{k'} = \left(C_4+\frac{K_BC_A}{K_A-1}\right)\boldsymbol w_A^T \boldsymbol h_{k'}+ \left(C_5+\frac{K_B}{(K_A-1)(1-C_4^2)}\right)\boldsymbol w_B^T \boldsymbol h_{k'}.
\end{equation}

Similarly, for the minority class $K_A+1\leq k\leq K$, combine (\ref{eq_minor}) with (\ref{condition_minor}), we obtain:

\begin{equation}
    C_6\boldsymbol h_k^T\boldsymbol h_{k'} = \left(C_4+\frac{K_A}{(K_B-1)(1-C_5^2)}\right)\boldsymbol w_A^T \boldsymbol h_{k'}+ \left(C_5+\frac{K_AC_B}{K_B-1}\right)\boldsymbol w_B^T \boldsymbol h_{k'}.
\end{equation}

In the above two equations, $k'=1,2,\cdots, K$. And we denote the coefficient of $\boldsymbol w_A$ in Eq. (\ref{condition_major}) and the coefficient of $\boldsymbol w_B$ in Eq. (\ref{condition_minor}) as $C_A$ and $C_B$ respectively for simplicity. After this deviation, we can find that both of the coefficients of $\boldsymbol w_A^T \boldsymbol h_{k'}$ and $\boldsymbol w_B^T \boldsymbol h_{k'}$ are constants. 

It can be obviously concluded that $\mathcal{NC}2$ and $\mathcal{NC}3$ do not hold under imbalanced dataset conditions. However, we can still compare the numerical differences between them under DEQ and fully connected layer. By adaptively specifying parameters $C_4$ and $C_5$, we can denote $(\boldsymbol h^0_k)^T\boldsymbol h^0_{k'}=\boldsymbol m_{k,k'}$. 

Thus, by considering all the equality conditions in (\ref{ineq6}), we can measure the distance from the features to the Simplex ETF.

\begin{equation}
    C_6\boldsymbol h_{k'}^T\boldsymbol h_k = C_4\boldsymbol h_{k'}^T\boldsymbol w_A+C_5\boldsymbol h_{k'}^T\boldsymbol w_B-\boldsymbol h_{k'}^T\boldsymbol w_k.
\end{equation}

\textbf{Case 1:} (Explicit fully connected layers)

\begin{equation}
    \begin{aligned}
        C_6\boldsymbol (h^0_{k'})^T\boldsymbol h_k &= C_4\boldsymbol W_\text{EX} \boldsymbol (h^0_{k'})^T\boldsymbol w_A+C_5\boldsymbol W_\text{EX} \boldsymbol (h^0_{k'})^T\boldsymbol w_B-\boldsymbol W_\text{EX} \boldsymbol (h^0_{k'})^T\boldsymbol w_k\\
        &=\boldsymbol W_\text{EX}\left(C_4\boldsymbol h_{k'}^T\boldsymbol w_A+C_5\boldsymbol h_{k'}^T\boldsymbol w_B-\boldsymbol h_{k'}^T\boldsymbol w_k\right)\\
        &\leq\frac{1}2\|\boldsymbol W_\text{EX}\|_F+\frac{1}2\left\|C_4\boldsymbol h_{k'}^T\boldsymbol w_A+C_5\boldsymbol h_{k'}^T\boldsymbol w_B-\boldsymbol h_{k'}^T\boldsymbol w_k\right\|\\
        &=E_H+\frac{1}2\boldsymbol M.
    \end{aligned}
\end{equation}

\textbf{Case 2:} (Deep Equilibrium Models)

\begin{equation}
    \begin{aligned}
        C_6\boldsymbol (h^0_{k'})^T\boldsymbol h_k &= C_4(\boldsymbol I-\boldsymbol W_\text{DEQ})^{-1} \boldsymbol (h^0_{k'})^T\boldsymbol w_A+C_5(\boldsymbol I-\boldsymbol W_\text{DEQ})^{-1} \boldsymbol (h^0_{k'})^T\boldsymbol w_B\\
        &\quad\quad\quad\quad\quad\quad\quad\quad-(\boldsymbol I-\boldsymbol W_\text{DEQ})^{-1}\boldsymbol (h^0_{k'})^T\boldsymbol w_k\\
        &=(\boldsymbol I-\boldsymbol W_\text{DEQ})^{-1}\left(C_4\boldsymbol h_{k'}^T\boldsymbol w_A+C_5\boldsymbol h_{k'}^T\boldsymbol w_B-\boldsymbol h_{k'}^T\boldsymbol w_k\right)\\
        &\leq\frac{1}2\|(\boldsymbol I-\boldsymbol W_\text{DEQ})^{-1}\|_F+\frac{1}2\left\|C_4\boldsymbol h_{k'}^T\boldsymbol w_A+C_5\boldsymbol h_{k'}^T\boldsymbol w_B-\boldsymbol h_{k'}^T\boldsymbol w_k\right\|\\
        &=\frac{1}{2(1-E_H)}+\frac{1}2\boldsymbol M.
    \end{aligned}
\end{equation}

Therefore, to compare these two models, we consider the case when the distance of these two models from the Simplex ETF is minimized. We denote each element in the Simplex ETF as $\boldsymbol s_{ij}$ and compare the differences between them. When the distance of DEQ is relatively smaller than that of explicit neural network, we can obtain:

\begin{equation}
\label{ineq8}
    \left|\frac{1}{2(1-E_H)}+\frac{1}2\boldsymbol m -\boldsymbol s\right|<\left|\frac{1}{2}E_H+\frac{1}2\boldsymbol m -\boldsymbol s\right|.
\end{equation}

For simplicity, we only consider the subscripts of $\boldsymbol s$ and $\boldsymbol m$, and denote $\mcircled{1}=\frac{1}{2(1-E_H)}+\frac{1}2\boldsymbol m -\boldsymbol s$ and $\mcircled{2}=\frac{1}{2}E_H+\frac{1}2\boldsymbol m -\boldsymbol s$. We then classify and discuss their magnitudes.

\begin{itemize}
    \item $\mcircled{1}>0$, $\mcircled{2}>0$:
    
    Since $\frac{1}{2(1-E_H)}>E_H$, there is a contradiction! Therefore, it does not hold.

    \item $\mcircled{1}>0$, $\mcircled{2}<0$: 
    
    $\frac{1}{2(1-E_H)}+\frac{1}2\boldsymbol m -\boldsymbol s<0$, $\frac{1}{2}E_H+\frac{1}2\boldsymbol m -\boldsymbol s>0$, which means $\frac{1}{2(1-E_H)}<E_H$. And that is a contradiction!

    \item $\mcircled{1}<0$, $\mcircled{2}>0$: 
    
    Since $\frac{1}{2(1-E_H)}+\frac{1}2\boldsymbol m -\boldsymbol s>0$, $\frac{1}{2}E_H+\frac{1}2\boldsymbol m -\boldsymbol s<0$, we have
    \begin{equation*}
        E_H<2\boldsymbol s-\boldsymbol m<\frac{1}{1-E_H}.
    \end{equation*}
    Besides, by the inequality $(\ref{ineq8})$, we have $\frac{1}2E_H+\frac{1}{2(1-E_H)}<2\boldsymbol s-\boldsymbol m$. Then find the intersection, we obtain that:
    \begin{equation*}
        \frac{1}2E_H+\frac{1}{2(1-E_H)}<2\boldsymbol s-\boldsymbol m<\frac{1}{1-E_H}.
    \end{equation*}

    \item $\mcircled{1}<0$, $\mcircled{2}<0$:

    Since $\frac{1}{2(1-E_H)}+\frac{1}{2}\boldsymbol m -\boldsymbol s<0$ and $\frac{1}{2}E_H+\frac{1}{2}\boldsymbol m -\boldsymbol s<0$, it implies that $\frac{1}{2(1-E_H)}>E_H$ always holds.
    
    Therefore, we only need to ensure that 
    \begin{equation*}
        2\boldsymbol s-\boldsymbol m>\min\left\{E_H,\frac{1}{1-E_H}\right\}.
    \end{equation*}
    
\end{itemize}

Combining these four cases and finding their intersection, we conclude that when the inequality
\begin{equation}
E_H < 2\boldsymbol s - \boldsymbol m < \frac{1}{1-E_H}
\end{equation}
is satisfied, the performance of DEQ is better than that of explicit neural network.

As for $\mathcal{NC}3$, consider the cosine distance with the feature $\boldsymbol h_k$ and $\boldsymbol w_k$.

\textbf{Case 1:} (Explicit fully connected layers)
\begin{equation}
\begin{aligned}
    \cos(\boldsymbol h_k,\boldsymbol w_k)_\text{EX}&=\frac{\boldsymbol h_k^T\boldsymbol w_k}{\|\boldsymbol w_k\|\|\boldsymbol h_k\|}\\
    &=\frac{\boldsymbol W_\text{EX}\left(\boldsymbol h_k^0\right)^T\boldsymbol w_k}{\|\boldsymbol w_k\|\|\boldsymbol W_\text{EX}\boldsymbol h_k^0\|}\\
    &\geq \frac{2\boldsymbol W_\text{EX}\left(\boldsymbol h_k^0\right)^T\boldsymbol w_k}{\|\boldsymbol w_k\|^2+\frac{1}2\|\boldsymbol W_\text{EX}\|^2+\frac{1}2\|\boldsymbol h_k^0\|}\\
    &\geq \frac{2E_H\left(\boldsymbol h_k^0\right)^T\boldsymbol w_k}{E_W+E_H}.
\end{aligned}
\end{equation}

\textbf{Case 2:} (Deep Equilibrium Model)

Very similarly, 

\begin{equation}
\begin{aligned}
    \cos(\boldsymbol h_k,\boldsymbol w_k)_\text{DEQ}&=\frac{\boldsymbol h_k^T\boldsymbol w_k}{\|\boldsymbol w_k\|\|\boldsymbol h_k\|}\\
    &=\frac{\left(\boldsymbol I-\boldsymbol W_\text{DEQ}\right)\left(\boldsymbol h_k^0\right)^T\boldsymbol w_k}{\|\boldsymbol w_k\|\|\left(\boldsymbol I-\boldsymbol W_\text{DEQ}\right)^{-1}\boldsymbol h_k^0\|}\\
    &\geq \frac{2\left(\boldsymbol I-\boldsymbol W_\text{DEQ}\right)\left(\boldsymbol h_k^0\right)^T\boldsymbol w_k}{\|\boldsymbol w_k\|^2+\frac{1}2\|\left(\boldsymbol I-\boldsymbol W_\text{DEQ}\right)\|^2+\frac{1}2\|\boldsymbol h_k^0\|}\\
    &\geq \frac{4E_H\left(\boldsymbol h_k^0\right)^T\boldsymbol w_k}{1+2(E_W+E_H)(1-E_H)}.
\end{aligned}
\end{equation}

If the performance of DEQ is better than explicit neural network, then we have \[\cos(\boldsymbol h_k,\boldsymbol w_k)_\text{DEQ}/\cos(\boldsymbol h_k,\boldsymbol w_k)_\text{exp}>1,\] which is equivalent to 
\begin{equation}
    \frac{E_H}{E_w+E_H}+E_H(1-E_H)<2.
\end{equation}

In summary, though DEQ does not completely mitigate the issue of minority collapse, it shows significant improvement compared to explicit neural network under some conditions that are relatively easy to satisfy in the manifestation of the $\mathcal{NC}$ phenomenon.

\section{More experiments}
\label{supp_experiment}

In this section, we provide more experimental results, including the $\mathcal{NC}$ phenomena of Explicit NN and DEQ, and the training results under other imbalanced conditions.

\begin{table*}[htbp]

\renewcommand\arraystretch{1.1}
\caption{Test Accuracy on Cifar-10 and Cifar-100 Dataset with $K_A=5$}
\vspace{5pt}
\centering
\resizebox{\textwidth}{!}{
\begin{tabular}{c|c|ccc|ccc}
\hline
                             &          & \multicolumn{3}{c|}{Cifar-10}                    & \multicolumn{3}{c}{Cifar-100}                    \\ \cline{2-8} 
                             & $R$        & 10             & 50             & 100            & 10             & 50             & 100            \\ \hline
\multirow{3}{*}{Explicit NN} & overall  & 80.73$\pm$0.48 & 63.08$\pm$0.87 & 44.86$\pm$1.43 & 52.62$\pm$0.86 & 41.62$\pm$0.68 & 37.33$\pm$2.29 \\
                             & majority & 94.18$\pm$0.56 & 91.02$\pm$0.89 & 89.32$\pm$0.79 & 74.10$\pm$1.03 & 73.94$\pm$0.25 & 74.24$\pm$1.13 \\
                             & minority & 67.80$\pm$0.35 & 35.14$\pm$0.65 & 0.40$\pm$3.86  & 31.10$\pm$0.70 & 9.30$\pm$1.10  & 0.42$\pm$3.04  \\ \hline
\multirow{3}{*}{DEQ}         & overall  & 81.36$\pm$1.03 & 65.03$\pm$1.90 & 46.09$\pm$1.77 & 53.31$\pm$0.98 & 44.07$\pm$2.04 & 39.11$\pm$2.46 \\
                             & majority & 93.14$\pm$1.81 & 90.88$\pm$2.83 & 90.20$\pm$0.85 & 72.90$\pm$1.65 & 75.98$\pm$1.75 & 75.79$\pm$0.96 \\
                             & minority & 69.58$\pm$0.66 & 39.18$\pm$1.46 & 1.26$\pm$4.93  & 33.72$\pm$0.79 & 12.16$\pm$3.75 & 2.42$\pm$5.89  \\ \hline
\end{tabular}
}
\label{tab_5}
\end{table*}

\begin{table*}[htbp]

\renewcommand\arraystretch{1.1}
\caption{Test Accuracy on Cifar-10 and Cifar-100 Dataset with $K_A=7$}
\vspace{5pt}
\centering
\resizebox{\textwidth}{!}{
\begin{tabular}{c|c|ccc|ccc}
\hline
            &          & \multicolumn{3}{c|}{Cifar-10}                    & \multicolumn{3}{c}{Cifar-100}                    \\ \cline{2-8} 
            & $R$      & 10             & 50             & 100            & 10             & 50             & 100            \\ \hline
\multirow{3}{*}{Explicit NN} & overall  & 83.17$\pm$0.40 & 66.91$\pm$0.39 & 53.27$\pm$0.81 & 59.11$\pm$0.84 & 51.71$\pm$1.02 & 50.72$\pm$0.60 \\ 
            & majority & 89.09$\pm$0.36 & 80.90$\pm$0.57 & 75.12$\pm$0.74 & 71.93$\pm$0.65 & 72.20$\pm$0.66 & 72.46$\pm$0.58 \\ 
            & minority & 69.37$\pm$0.49 & 34.26$\pm$0.30 & 2.30$\pm$1.01  & 29.20$\pm$0.92 & 3.90$\pm$1.29  & 0.00$\pm$0.00  \\ \hline
\multirow{3}{*}{DEQ}         & overall  & 83.78$\pm$1.85 & 69.47$\pm$1.86 & 56.74$\pm$0.98 & 60.51$\pm$0.88 & 52.99$\pm$1.86 & 51.79$\pm$0.92 \\ 
            & majority & 88.98$\pm$1.99 & 82.91$\pm$2.22 & 78.81$\pm$0.67 & 72.90$\pm$1.19 & 72.99$\pm$0.98 & 73.98$\pm$0.66 \\ 
            & minority & 71.65$\pm$1.63 & 38.12$\pm$1.61 & 5.20$\pm$1.91  & 31.13$\pm$0.83 & 6.33$\pm$2.35  & 0.00$\pm$0.00  \\ \hline
\end{tabular}
}
\label{tab_7}
\end{table*}

\begin{figure}[H]
    \centering
    \includegraphics[scale=0.35]{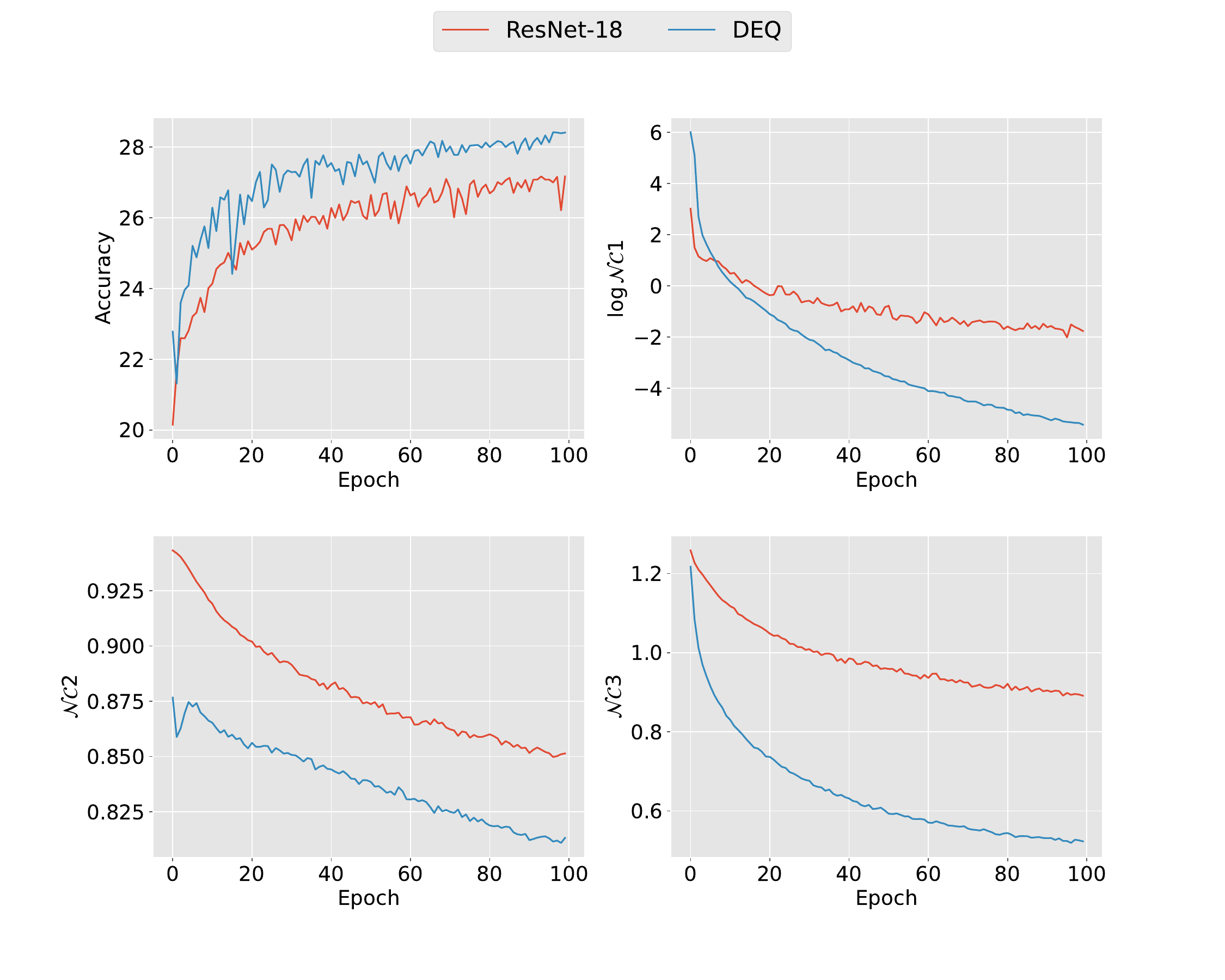}
    \caption{Accuracy and $\mathcal{NC}$ phenomenon on imbalanced dataset with $K_A=3$, $K_B=7$, $R=100$}
\end{figure}

\begin{figure}[H]
    \centering
    \includegraphics[scale=0.35]{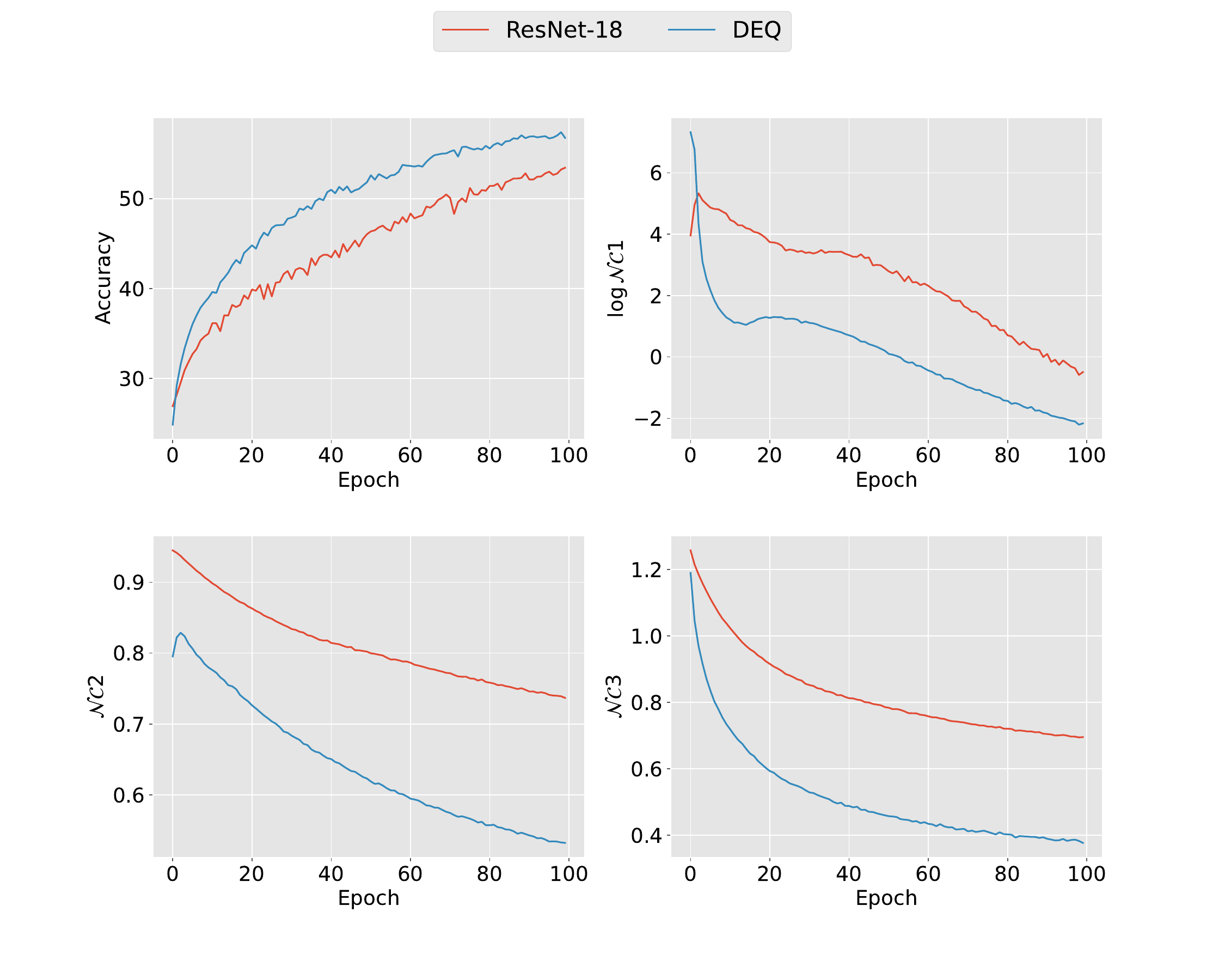}
    \caption{Accuracy and $\mathcal{NC}$ phenomenon on imbalanced dataset with $K_A=7$, $K_B=3$, $R=100$}
\end{figure}

\clearpage
\newpage
\begin{enumerate}

\item {\bf Claims}
    \item[] Question: Do the main claims made in the abstract and introduction accurately reflect the paper's contributions and scope?
    \item[] Answer: \answerYes{} 
    \item[] Justification: As stated in the abstract and introduction, this paper is the first to analyze the representation of the Deep Equilibrium Model from the perspective of Neural Collapse, accurately reflecting the key contributions and scope. 

\item {\bf Limitations}
    \item[] Question: Does the paper discuss the limitations of the work performed by the authors?
    \item[] Answer: \answerYes{} 
    \item[] Justification: The limitation analysis is provided in the Conclusion. 

\item {\bf Theory Assumptions and Proofs}
    \item[] Question: For each theoretical result, does the paper provide the full set of assumptions and a complete (and correct) proof?
    \item[] Answer: \answerYes{} 
    \item[] Justification: The assumptions and proofs are provided in appendix B and C. 

    \item {\bf Experimental Result Reproducibility}
    \item[] Question: Does the paper fully disclose all the information needed to reproduce the main experimental results of the paper to the extent that it affects the main claims and/or conclusions of the paper (regardless of whether the code and data are provided or not)?
    \item[] Answer: \answerYes{} 
    \item[] Justification: These details are provided in Section 5.1 - Experiment setup.

\item {\bf Open access to data and code}
    \item[] Question: Does the paper provide open access to the data and code, with sufficient instructions to faithfully reproduce the main experimental results, as described in supplemental material?
    \item[] Answer: \answerYes{} 
    \item[] Justification: We will release the code once the paper is accepted.

\item {\bf Experimental Setting/Details}
    \item[] Question: Does the paper specify all the training and test details (e.g., data splits, hyperparameters, how they were chosen, type of optimizer, etc.) necessary to understand the results?
    \item[] Answer: \answerYes{} 
    \item[] Justification: These details are provided in Section 5.1 - Experiment setup.

\item {\bf Experiment Statistical Significance}
    \item[] Question: Does the paper report error bars suitably and correctly defined or other appropriate information about the statistical significance of the experiments?
    \item[] Answer: \answerYes{} 
    \item[] Justification: The standard deviation in our experimental results (Table 1-4) shows the statistical significance. 

\item {\bf Experiments Compute Resources}
    \item[] Question: For each experiment, does the paper provide sufficient information on the computer resources (type of compute workers, memory, time of execution) needed to reproduce the experiments?
    \item[] Answer: \answerYes{} 
    \item[] Justification: The computer resources are provided in Section 5.1 - Experiment setup. 
    
\item {\bf Code Of Ethics}
    \item[] Question: Does the research conducted in the paper conform, in every respect, with the NeurIPS Code of Ethics \url{https://neurips.cc/public/EthicsGuidelines}?
    \item[] Answer: \answerYes{} 
    \item[] Justification: Our research conformed with the NeurIPS Code of Ethics.

\item {\bf Broader Impacts}
    \item[] Question: Does the paper discuss both potential positive societal impacts and negative societal impacts of the work performed?
    \item[] Answer: \answerNA{} 
    \item[] Justification: Our paper primarily focuses on theoretical research in machine learning, comparing two typical neural network algorithms, with no relevance to societal impacts.
    
\item {\bf Safeguards}
    \item[] Question: Does the paper describe safeguards that have been put in place for responsible release of data or models that have a high risk for misuse (e.g., pretrained language models, image generators, or scraped datasets)?
    \item[] Answer: \answerNA{} 
    \item[] Justification: Our paper uses standard CIFAR-10 and CIFAR-100 datasets, which do not involve such issues.

\item {\bf Licenses for existing assets}
    \item[] Question: Are the creators or original owners of assets (e.g., code, data, models), used in the paper, properly credited and are the license and terms of use explicitly mentioned and properly respected?
    \item[] Answer: \answerYes{} 
    \item[] Justification: The datasets CIFAR-10 and CIFAR-100 are cited properly. Other assets are not applied in this paper. 

\item {\bf New Assets}
    \item[] Question: Are new assets introduced in the paper well documented and is the documentation provided alongside the assets?
    \item[] Answer: \answerNA{} 
    \item[] Justification: This paper does not introduce new assets.

\item {\bf Crowdsourcing and Research with Human Subjects}
    \item[] Question: For crowdsourcing experiments and research with human subjects, does the paper include the full text of instructions given to participants and screenshots, if applicable, as well as details about compensation (if any)? 
    \item[] Answer: \answerNA{} 
    \item[] Justification: Our paper does not involve crowdsourcing nor research with human subjects. 

\item {\bf Institutional Review Board (IRB) Approvals or Equivalent for Research with Human Subjects}
    \item[] Question: Does the paper describe potential risks incurred by study participants, whether such risks were disclosed to the subjects, and whether Institutional Review Board (IRB) approvals (or an equivalent approval/review based on the requirements of your country or institution) were obtained?
    \item[] Answer: \answerNA{} 
    \item[] Justification: Our paper does not involve crowdsourcing nor research with human subjects.

\end{enumerate}

\end{document}